\documentclass[a4paper,10pt]{article}
\usepackage[utf8]{inputenc}
\usepackage{fullpage}
\usepackage{amsmath}
\usepackage{amsfonts}
\usepackage{amsthm}
\usepackage{amssymb}
\usepackage{xcolor}
\usepackage{graphicx}
\graphicspath{{.}}
\usepackage{algorithm}
\usepackage{xspace}
\usepackage{algpseudocode}
\usepackage{enumitem}
\setlist{leftmargin=5.5mm}

\usepackage{caption}
\usepackage{subcaption}

\numberwithin{equation}{section}

\newtheorem{definition}{Definition}[section]
\newtheorem*{example}{Example}
\newtheorem{remark}[definition]{Remark}

\newtheorem{lemma}[definition]{Lemma}
\newtheorem{corollary}[definition]{Corollary}
\newtheorem{proposition}[definition]{Proposition}
\newtheorem{theorem}[definition]{Theorem}

\allowdisplaybreaks
\newcommand{\N}{\mathbb{N}}

\newcommand{\R}{\mathbb{R}}
\newcommand{\sgn}{\operatorname{sign}}
\newcommand{\eu}[1]{\left\lVert#1\right\rVert_2}

\newcommand{\Pb}{\mathbb{P}}
\newcommand{\E}{\mathbb{E}}

\newcommand{\del}{\delta}
\newcommand{\eps}{\varepsilon}

\newcommand{\argmin}{\operatorname{arg\,min}}

\usepackage[colorlinks, linkcolor = black, citecolor = black ]{hyperref} 

\newcommand{\bs}{\boldsymbol}
\newcommand{\cl}{\mathcal}
\newcommand{\bb}{\mathbb}

\newcommand{\ts}{\textstyle}
\newcommand{\clip}{\cl C}
\newcommand{\scp}[3][]{#1\langle #2, #3 #1\rangle}

\newcommand{\ie}{\emph{i.e.}, }
\newcommand{\eg}{\emph{e.g.}, }
\newcommand{\rv}{%
  \ifmmode
  \text{\emph{r.v.}}%
  \else%
  \emph{r.v.}\@\xspace%
  \fi%
}

\newcommand{\rvs}{%
  \ifmmode
  \text{\emph{r.v.}s}%
  \else%
  \emph{r.v.}s\@\xspace%
  \fi%
}

\newcommand{\ouralgo}{FO-SGD\xspace}
\newcommand{\bnabla}{{\nabla}}
\newcommand{\Htransform}{H}
\newcommand{\encoder}{{\mathcal{E}}}
\newcommand{\decoder}{{\mathcal{D}}}
\newcommand{\dithQ}{{\mathcal Q}}
\newcommand{\server}{{\rm s}}
\newcommand{\oraclegrad}{{g}}
\newcommand{\servergrad}{\bar{\bs g}}
\newcommand{\serverqgrad}{\bar{\bs q}}
\newcommand{\workerdgrad}{\bs g^{\rm d}}
\newcommand{\Ztransform}{{\cl Z}}
\newcommand{\Lfct}{{\cl L}}

\title{Flattened one-bit stochastic gradient descent:\\ 
  compressed distributed optimization with controlled variance}

\author{Alexander Stollenwerk and Laurent Jacques\footnote{Part of this research was supported by the Fonds de la Recherche Scientifique – FNRS under Grant T.0136.20 (Project Learn2Sense).}\\  
  \small INMA/ICTEAM, UCLouvain, Belgium\\
\normalsize Technical report: TR.2024.01}

\begin{document}

\maketitle

\begin{abstract}
We propose a novel algorithm for distributed stochastic gradient descent (SGD) with compressed gradient communication in the parameter-server framework. Our gradient compression technique, named flattened one-bit stochastic gradient descent (\ouralgo), relies on two simple algorithmic ideas: \emph{(i)} a one-bit quantization procedure leveraging the technique of dithering, and \emph{(ii)} a randomized fast Walsh-Hadamard transform to flatten the stochastic gradient before
quantization. As a result, the approximation of the true gradient in this scheme is biased, 
but it prevents commonly encountered algorithmic problems, such as exploding variance in the one-bit compression regime, deterioration of performance in the case of sparse gradients, and restrictive assumptions on the distribution of the stochastic gradients.   
In fact, we show SGD-like convergence guarantees under mild conditions. 
The compression technique can be used in both directions of worker-server communication, therefore admitting distributed optimization with full communication compression.
\end{abstract}

\section{Introduction}

In recent years, distributed optimization has become a crucial 
ingredient of a wide range of machine learning tasks \cite{TL12, RK16, SF18}.
Particularly in deep learning \cite{LB15, S15}, where ever-larger neural networks are being trained using ever-larger datasets, the process of training 
can be substantially accelerated by distributing both the data and the optimization procedure over many computing \emph{workers}.
The key feature of these types of optimization problems, which makes this parallelization possible, is that the objective function is structured as a sum over local observations made on a given computing unit. 

{In the case of a differentiable objective function,} a popular approach to solve the optimization problem {with first-order methods---such as (stochastic) gradient descent---}is to spread the gradient calculation over several workers~\cite{LA14}. The workers transmit their gradients to a parameter server, which calculates the average of the received gradients and sends it back to the workers. The workers then update the model parameters via a gradient descent step using the same (global) gradient.
While the distribution of gradient calculation across many workers generally leads to a great speed-up of training, the cost of gradient communication (from worker to server and vice versa) often constitutes a significant bottleneck in applications. {Moreover, selecting an appropriate lossy compression to limit the gradient communication bandwidth may impact---or even break---the convergence guarantees of the minimization method.}

Alongside decentralized approaches \cite{NO09, WO12, KS19}, 
one effective way to meet {these challenges} is to apply so-called \emph{gradient compression techniques}. {For these,} interesting results have been recently established, which can roughly be grouped in two directions. In the first direction, the idea is to apply randomized compression schemes in such a way that the resulting compressed gradient is an unbiased estimator of the true gradient \cite{AG17, WX17}. As a result, interesting trade-offs between communication cost and convergence guarantees have been demonstrated. An important drawback of these schemes is that 
reduction in communication cost comes at the price of drastically increased variance
bounds{, associated with the very variance of the compressed gradient estimator.} 
In the second direction, the authors of \cite{BW18, BZ19} propose to entrywise quantize the gradients using the sign function. 
This simple procedure can be applied both for worker-server and server-worker communication, thus implying distributed optimization with full one-bit communication compression. 
The proven convergence guarantees show that this scheme is even able to outperform distributed SGD in certain scenarios. However, 
if gradients before quantization are sparse, then the proven guarantees become drastically weaker. 

{In this work, w}e propose a novel communication efficient algorithm for distributed SGD, which \emph{(i)} allows for compression in both directions of worker-server communication, \emph{(ii)} avoids variance explosion, and \emph{(iii)} achieves SGD-like convergence guarantees under mild assumptions.

\section{Motivating example}
In this small section we motivate the usage of the flattening trick by showing that the algorithm signSGD will fail to converge in situations where the stochstic gradient is sparse - even if the underlying optimization problem is convex. 

Assume that the stochastic gradient oracle for $f$ is $1$-sparse and $d\geq 2$. If we are in a scenario with $N$ workers and each worker draws an independent copy $\bs g_{t,n}$ of the stochastic gradient oracle at the current optimization point $\bs x_t\in \R^d$, then the update rule of signSGD with majority vote reads 
$$\bs x_{t+1}= \bs x_t - \del_t \bs v_t,$$
where
$$\bs v_t = \sgn(\sum_{n=1}^N \sgn( \bs g_{t,n})).$$
We claim that 
$\langle \bs v_t, \bs 1 \rangle\geq 0$, where $\bs 1\in \R^d$ denotes the vector with all coordinates equal to $1$. Indeed, set $\bs q_{t,n}=\sgn(\bs g_{t,n})$ and observe that if there exists a coordinate $i\in [n]$, where $\big(\sum_{n=1}^N \bs q_{t,n})\big)_i<0$, then there exist two disjoint sets $A, B\subset [N]$ with $A\cup B = [N]$ such that $(\bs q_{t,n})_i=-1$ for all $n\in A$ and $(\bs q_{t,n})_i=1$ for all $n\in B$ and such that $|A|>|B|$. But since $\bs q_{t,n}=\sgn(\bs g_{t,n})$ and $\bs g_{t,n}$ is $1$-sparse, 
this implies that $(\bs g_{t,n})_j=0$ for all $j\neq i$ and for all $n\in A$. Therefore, 
$(\bs q_{t,n})_j=1$ for all $j\neq i$ and for all $n\in A$, which implies that $\big(\sum_{n=1}^N \bs q_{t,n}\big)_j \geq 0$ for all $j\neq i$. In total, we obtain that the vector $\sum_{n=1}^N \bs q_{t,n}=\sum_{n=1}^N \sgn( \bs g_{t,n})$ has all entries greater or equal to zero, except for at most one entry. This shows $\langle \sgn(\sum_{n=1}^N \sgn( \bs g_{t,n})), \bs 1 \rangle= d-1 \pm 1$, which is non-negative if $d\geq 2$. In particular, this shows that in every iteration round, signSGD with majority vote can only move into the direction of the halfspace $H_{ -\bs 1}=\{\bs v\in \R^d \,|\, \langle \bs v, -\bs 1\rangle\geq 0\}$, which will of course prevent convergence for many (convex) optimization problems. We remark that a similar conclusion holds true if the gradient descent vector $\bs v_t$ takes the form 
$$\bs v_t = \tfrac{1}{N}\sum_{n=1}^N \sgn( \bs g_{t,n}).$$
Indeed, using that $\langle \sgn( \bs g_{t,n}), \bs 1 \rangle= d-1 \pm 1\geq 0$ (due to the $1$-sparsity of $\bs g_{t,n}$), we obtain
$$\langle \tfrac{1}{N}\sum_{n=1}^N \sgn( \bs g_{t,n}), \bs 1 \rangle = \tfrac{1}{N}\sum_{n=1}^N \langle \sgn( \bs g_{t,n}), \bs 1 \rangle\geq 0.$$
In conclusion, the fact that signSGD cannot explore certain directions if the underlying stochastic gradient is sparse, does not come from the majority vote, but is due to the combination of sparsity and $\sgn$-quantization. 

In order to circumvent these issues, we make use of the flattening trick, which applies a randomized universal sensing basis $\bs \Htransform_\eps$ to the sparse gradient $\bs g_{t,n}$ in order to spread out the energy of the resulting vector $\bs \Htransform_\eps \bs g_{t,n}$ rather evenly across its coordinates with high probability. This will ensure that after $1$-bit quantization, the resulting vector can take many different values in $\{-1,1\}^d$, thereby allowing the associated algorithm (\ouralgo) to explore more directions in $\R^d$.

\section{Preliminaries}


Before explaining our algorithm in Sec.~\ref{sec:fo-sgd}, let us first introduce several useful preliminary concepts. These include the very definition of our distributed computation framework, the compression of communication in this distributed setting, as well as a specific flatting trick for efficient gradient-compression.

\paragraph{Distributed stochastic gradient descent}

Let us consider a differentiable function\footnote{Henceforth, for simplicity, we assume $f$ differentiable over $\bb R^d$, but our work can be adapted to a smaller differentiability domain.} $f:\bb R^d \to \bb R$ that we wish to minimize, that is, we want to find the solution to 
$$
\bs x_\ast \in \argmin_{\bs x} f(\bs x). 
$$
Given a stochastic gradient $\tilde{g}: \bb R^d \to \bb R^d$ such that $\bb E \tilde g(\bs x) = \nabla f(\bs x)$ {for all $\bs x \in\R^d$}, the stochastic gradient descent (SGD) algorithm amounts to iterating $T$ times the following procedure from some initialization $\bs x_0 \in \bb R^d$ (see \eg \cite[Sec. 5.3]{Bach23}): 
\begin{equation}\label{eq:SGD}
    \bs x_{t+1} = \bs x_t - \del_t\, \tilde{g}(\bs x_t), \quad\text{for}\ t=0,\ldots,T,
\end{equation} 
where $\bs x_t$ is the current minimization estimate, and $\del_t$ is a variable step size in the $t$-th iteration. Note that here and in the rest of this work, (random) vectors and (random) matrices (but not functions with vector or matrix image) are denoted by bold symbols, while scalar quantities are written with light symbols.

As \cite{AG17,BW18,KS19}, our approach aims at both distributing the SGD algorithm between a set of workers---separately estimating multiple instances of the gradient---and a central server---gathering the workers gradient and updating the global minimization---, and compressing the communication of information (such as the current estimate) between these agents and the server. We use for this the notion of \emph{stochastic gradient oracle} introduced in, \eg \cite{AG17}, which each worker can appeal to.

\begin{definition}\label{def:stochastic_gradient_oracle}
For a function $\oraclegrad: \Omega\times \R^d\to \R^d$ and a random variable $\bs \xi\in \Omega$, we call $\oraclegrad(\bs \xi, \cdot)$ a \emph{stochastic gradient oracle} for a differentiable function $f: \bb R^d \to \bb R$ if 
\begin{equation*}
\E \big[\oraclegrad(\bs \xi,\bs x)\big] = \bnabla f(\bs x)\quad \text{ for all } \bs x\in \R^d.
\end{equation*}
Further, we will assume that the variance of $\oraclegrad(\bs \xi,\cdot)$ is bounded by $\sigma^2$, \ie
\begin{equation*}
        \E \big[ \eu{\oraclegrad(\bs \xi, \bs x) - \E \big[\oraclegrad(\bs \xi, \bs x)\big]}^2\big] \leq \sigma^2\quad \text{ for all } \bs x\in \R^d.
\end{equation*}
\end{definition}

Let us give a brief example of a stochastic gradient oracle in the case of linear regression.

\begin{example}[Linear regression via least-squares estimation]
Let $\bs A=(\bs a_1, \ldots, \bs a_m)^\top$  be an $m \times d$ matrix of input variables $\bs a_i\in \R^d$ and $\bs y = (y_1, \ldots, y_m)^\top \in \R^m$ a vector of observed values $y_i$, with $i=1,\ldots, m$. Solving the corresponding linear regression problem via least-squares estimation amounts to minimizing the convex objective function 
\begin{equation*}
\ts f(\bs x)=\frac{1}{m}\eu{\bs y-\bs A\bs x}^2=\frac{1}{m}\sum_{i=1}^m(y_i-\scp{\bs a_i}{\bs x})^2 \qquad \text{ for } \bs x\in \R^d.
\end{equation*}
The gradient of $f$ reads
\begin{equation*}
	\ts \bnabla f(\bs x) = \tfrac{1}{m} \sum_{i=1}^m 2(\scp{\bs a_i}{\bs x}-y_i)\bs a_i.
\end{equation*}
For $m'\in \N$, let us denote by $[m]$ the index set $\{1,\ldots,m\}$, and define the random vector $\bs \xi\in [m]^{m'} := \{1,\ldots,m\}^{m'}$ with independent coordinates $\xi_i$, which are uniformly distributed on $[m]$, which we compactly note as $\bs \xi \sim \cl U([m]^{m'})$. It is straightforward to verify that 
\begin{equation}
	\ts \oraclegrad(\bs \xi,\bs x)=\tfrac{1}{m'}\sum_{i=1}^{m'}2(\scp{\bs a_{\xi_i}}{\bs  x}-y_{\xi_i})\bs a_{\xi_i}
\end{equation}
is a stochastic gradient oracle for $f$. We can also observe that, by Lemma~\ref{lem:var_decrease}, the variance of $\oraclegrad(\bs \xi,\cdot)$ decays as $\mathcal{O}(\tfrac{1}{m'})$. \hfill $\diamond$
\end{example}

\paragraph{Quantized communications}

In this work, similarly to \cite{BW18,BZ19,AG17}, we target the definition of a distributed SGD with severely quantized communications between the workers and the server. To this aim, all communications are subject to the application of the following (parametrisable) one-bit quantizer.  
\begin{definition}[$K$-averaged dithered one-bit quantizer]\label{def:K_bit_quantizer}
For $K\in \N$, let us introduce $K$ independent random vectors, or \emph{dithers}, $\bs \tau_1, \ldots, \bs \tau_K\in \R^d$, all uniformly distributed in $[-\lambda, \lambda]^d$ for a \emph{dithering amplitude} $\lambda>0$. Define the \emph{$K$-averaged dithered one-bit quantizer} $\dithQ_{\lambda, K}:\R^d\to [-\lambda, \lambda]^d$ by
\begin{equation*}
\dithQ_{\lambda, K}(\bs x) := \tfrac{\lambda}{K}\sum_{i=1}^K\sgn(\bs x+\bs \tau_i)\quad \text{ for all } \bs x\in \R^d.
\end{equation*}
Here, the $\sgn$-function amounts to applying the sign operator component-wise. If $K=1$, we will simply write $\dithQ_{\lambda}$
instead of $\dithQ_{\lambda, 1}$.
\end{definition}

In this definition, the quantization process is associated with a \emph{dithering procedure}. {The technique of dithering was first introduced in \cite{roberts_picture_1962} in order to remove artefacts from quantized pictures. Moreover, this technique is well-known in the field of quantized compressed sensing, where it is 
successfully applied in order to reconstruct signals from strongly quantized measurements (e.g. see \cite{JC17, DM18, XJ19, JMPS21}). } 

\begin{remark}
Since $\tfrac{K}{\lambda}\cdot \dithQ_{\lambda, K}(\bs x)\in \{-K, -K+2, \ldots, K-2, K\}^d$, we can store $\tfrac{K}{\lambda}\cdot \dithQ_{\lambda, K}(\bs x)$ using at most $\log_2(K+1)\cdot d$ bits. 
\end{remark}

\paragraph{The flattening trick}

 {Anticipating further technical developments provided in Sec.~\ref{sec:fo-sgd}, in our approach, the compression of information between the workers and the server is sustained by a specific transformation whose aim is to``flatten" the gradient before a necessary quantization process. In \cite{BW18,BZ19}, it was indeed observed that if gradients
before quantization are sparse, then the convergence guarantees of the proposed distributed optimization become drastically weaker. Our objective is here to prevent, with high probability, this sparse configuration to occur by pre-encoding the gradient before its quantization, so that to flatten its possibly sparse content, and next to decode the quantize output to recover a gradient estimate. This flattening is first supported by the use of a \emph{universal basis}, which includes the discrete Fourier and Walsh-Hadamard bases.}  

\begin{definition}\label{def:universal_basis}(Universal sensing basis \cite{puy2012universal})
An orthonormal matrix $\bs \Htransform\in \R^{d\times d}$ is called a universal sensing basis if $|\bs \Htransform_{k,l}|=\tfrac{1}{\sqrt{d}}$ for all $k,l\in [d]$. 
\end{definition}

For instance, if $d= 2^k$ for some integer $k$, the matrix $\tfrac{1}{\sqrt{d}}\bs{\mathcal{H}}_k\in \R^{d\times d}$ is a universal sensing basis 
associated with the (fast) Walsh-Hadamard transform $\bs{\mathcal{H}}_k$ iteratively defined by $\bs{\mathcal{H}}_0=1$ and 
\begin{equation*}
\bs{\mathcal{H}}_k := 
\begin{bmatrix}
\bs{\mathcal{H}}_{k-1} & \;\;\; \bs{\mathcal{H}}_{k-1} \\
\bs{\mathcal{H}}_{k-1} & -\bs{\mathcal{H}}_{k-1} 
\end{bmatrix} = \begin{bmatrix}
1 & 1\\1&-1 
\end{bmatrix} \otimes \bs{\mathcal{H}}_{k-1},
\end{equation*}
with $\otimes$ the Kronecker product. If we extend Def.~\ref{def:universal_basis} to the complex field, the discrete Fourier matrix is another example of universal basis \cite{puy2012universal}. 

{Our procedure also needs a randomization of the previous transform, which is affiliated to the technique of ``spread spectrum" \cite{puy2012universal}.}

\begin{definition}\label{def:rand_universal_basis}(Randomized universal sensing basis)
Let $\bs \Htransform\in \R^{d\times d}$ be a universal sensing basis and $\bs \eps\in \{-1,1\}^d$ uniformly distributed. 
We call the random matrix $\bs \Htransform_{\bs \eps}:=\bs \Htransform \bs D_{\bs \eps}$ a \emph{randomized universal sensing basis}
associated to the universal sensing basis $\bs \Htransform$. 
Observe that its inverse $\bs \Htransform_{\bs \eps}^{-1}$ takes the form $\bs \Htransform_{\bs \eps}^{-1}=\bs D_{\bs \eps} \bs \Htransform^\top$.
\end{definition}

The following well-known result (\eg see \cite[Lemma B.2]{Oym16}) shows the impact of a randomized universal basis on the components of a given vector.

\begin{lemma}\label{lem:Hadamard}
Let $\bs \Htransform_\eps\in \R^{d\times d}$ be a randomized universal sensing basis.
For any $\bs x\in \R^d$ and $\alpha \geq 2$,  
\begin{equation*}
\|\bs \Htransform_{\bs \eps} \bs x\|_\infty\leq \alpha \sqrt{\tfrac{\log d}{d}}\eu{\bs x}
\end{equation*}
with probability at least $1-2\exp(-\tfrac{1}{4}\alpha^2\log d)$.
\end{lemma}

For convenience, we included a proof of Lemma~\ref{lem:Hadamard} in the appendix.
In the following, we will fix a universal sensing basis $\bs \Htransform$ and suppress the dependence of $\bs \Htransform_{\bs \eps}$ on $\bs \Htransform$.
We can now provide the definition of the gradient encoding and decoding scheme.

\begin{definition}[Encoder]\label{eq:encoder}
Let $\bs \Htransform_{\bs \eps}\in \R^{d\times d}$ be a randomized universal sensing basis and $\dithQ_{\lambda, K}$ an (independent) $K$-averaged dithered one-bit quantizer. Define the map $\encoder_{\bs \eps, \lambda, K}:\R^d\to \{-K, -K+2, \ldots, K-2, K\}^d$ by
\begin{equation*}
\encoder_{\eps, \lambda, K}(\bs x):=\tfrac{K}{\lambda}\cdot \dithQ_{\lambda, K}(\bs \Htransform_{\bs \eps} \bs x)\quad \text{ for all } \bs x\in \R^d.
\end{equation*}
We call $\encoder_{\bs \eps, \lambda, K}$ an \emph{$(\bs \eps, \lambda, K)$-encoder}. If $K=1$, we simply write 
$\encoder_{\bs \eps, \lambda}$ instead of $\encoder_{\bs \eps, \lambda, 1}$ and call it an \emph{$(\bs \eps, \lambda)$-encoder}.
\end{definition}

\begin{definition}[Decoder]\label{eq:decoder}
Let $\bs \Htransform_{\bs \eps}\in \R^{d\times d}$ be a randomized universal sensing basis. Define the map $\decoder_{{\bs \eps}, \lambda, K}:\R^d\to \R^d$ by 
\begin{equation*}
\decoder_{{\bs \eps}, \lambda, K}(\bs x) := \tfrac{\lambda}{K}\cdot \bs \Htransform_{\bs \eps}^{-1} \bs x \quad \text{ for all } \bs x\in \R^d.
\end{equation*}
We call $\decoder_{{\bs \eps}, \lambda, K}$ an \emph{$({\bs \eps}, \lambda, K)$-decoder}. If $K=1$, we simply write 
$\decoder_{{\bs \eps}, \lambda}$ instead of $\decoder_{{\bs \eps}, \lambda, 1}$ and call it an \emph{$({\bs \eps}, \lambda)$-decoder}.
\end{definition}

Let us point out that both the randomized Hadamard transform as well as the dithered quantizer \eqref{def:K_bit_quantizer} are
well known tools in signal processing and compressed sensing (see \eg \cite{OT17, DM18, XJ19, DMS22}).

\section{\ouralgo Algorithm}
\label{sec:fo-sgd}

The goal of the algorithm is to minimize a differentiable function $f:\R^d\to \R$.  In the distributed setting with $N$ workers, every worker has access to a stochastic gradient oracle $\oraclegrad(\bs \xi, \cdot):\R^d\to \R^d$ for $f$. 
The overall procedure is similar to (centralized) distributed SGD algorithm: at every iteration $t$, every worker knows the current optimization point $\bs x_t$ and computes an independent copy $\oraclegrad_{t,n}\in \R^d$  of the stochastic gradient oracle $\oraclegrad(\bs \xi, \bs x_t)$. It sends its computed gradient $\oraclegrad_{t,n}$ to a common parameter server, which then averages the received gradients and sends the computed average back to all workers. All workers then execute a gradient descent step using the same averaged gradient to obtain the next optimization point $\bs x_{t+1}$. 

In order to reduce communication cost which comes with sending full precision gradients between workers and server, distributed \ouralgo proposes to additionally compress gradients in every round of communication. More specifically, every round of communication, first from workers to server and then from server to workers, consists of \emph{(1)} an encoding (or compression) step, \emph{(2)} a transmission step, and \emph{(3)} a decoding (decompression) step.
Let us describe the optimization procedure of \ouralgo in full detail:
\medskip

\noindent\underline{\emph{Initialization:}} The optimization procedure is initialized at some starting point $\bs x_0\in \R^d$, which is known to all workers. 

\medskip
\noindent\underline{\emph{At iteration $t$:}} 
\begin{itemize}
\item Current minimization estimate: $\bs x_t\in \R^d$ 
\item  {Workers computation:} Every worker $n\in [N]$ 
computes an independent copy $\bs \oraclegrad_{t,n}\in \R^d$ of the stochastic gradient oracle $\oraclegrad(\bs \xi, \bs x_t)$.
\item Workers-to-server communication:
\begin{itemize}
\item[1)] \emph{Encoding:} Every worker $n\in [N]$ encodes its gradient $\bs \oraclegrad_{t,n}$ 
by (independently) drawing a uniformly distributed vector ${\bs \eps}_{t,n}\in \{-1,1\}^d$ and applying the associated $({\bs \eps}_{t,n},\lambda)$-encoder $\encoder_{{\bs \eps}_{t,n}, \lambda}$ (see Def.~\ref{eq:encoder}) to~$\oraclegrad_{t,n}$. The resulting vector 
$$\ts \bs q_{t,n}=\encoder_{{\bs \eps}_{t,n}, \lambda}(\bs \oraclegrad_{t,n})$$
can be stored using $d$ bits. 
\item[2)] \emph{Transmission:} All workers $n\in [N]$ send $\bs q_{t,n}$ and ${\bs \eps}_{t,n}$ to the server. 
\item[3)] \emph{Decoding:} For every $n\in [N]$, the server uses the string ${\bs \eps}_{t,n}$ to
decode $\bs q_{t,n}$ by applying the associated $({\bs \eps}_{t,n},\lambda)$-decoder $\decoder_{{\bs \eps}_{t,n}, \lambda}$ (see Def.~\ref{eq:decoder}) to $\bs q_{t,n}$. The resulting vector 
$$\ts \tilde{\bs \oraclegrad}_{t,n}=\decoder_{{\bs \eps}_{t,n}, \lambda}(\bs q_{t,n})$$ 
is an approximation of the original (uncompressed) gradient $\bs \oraclegrad_{t,n}$.
\end{itemize}
\item Server computation: The server computes the average of the decoded gradients:
\begin{equation}\label{eq:server_average}
\ts \servergrad_t = \tfrac{1}{N}\sum_{n=1}^N \tilde{\bs \oraclegrad}_{t,n}.
\end{equation}
\item Server-to-workers communication: 
\begin{itemize}
\item[1)] \emph{Encoding:} The server encodes $\servergrad_t$ 
by (independently) drawing a uniformly distributed vector ${\bs \eps}_{t}^{\server}\in \{-1,1\}^d$ and applying the associated $({\bs \eps}_{t}^{\server},\lambda^\server, K)$-encoder $\encoder_{{\bs \eps}_{t}^{\server}, \lambda^\server, K}$ (see Def.~\ref{eq:encoder}) to $\servergrad_t$. The resulting vector $$\ts \serverqgrad_t=\encoder_{{\bs \eps}_{t}^{\server}, \lambda^\server, K}(\servergrad_t)$$ can be stored using $\log_2(K+1)\cdot d$ bits. 
\item[2)] \emph{Transmission:} The server sends $\serverqgrad_t$ and ${\bs \eps}_{t}^{\server}$ to every worker.
\item[3)] \emph{Decoding:} Every worker uses the string ${\bs \eps}_{t}^{\server}$ to
decode $\serverqgrad_t$ by applying the associated $({\bs \eps}_{t}^{\server},\lambda^\server, K)$-decoder $\decoder_{{\bs \eps}_{t}^{\server}, \lambda^\server, K}$ (see Def.~\ref{eq:decoder}) to $\serverqgrad_t$. The resulting vector $$\ts \workerdgrad_t=\decoder_{{\bs \eps}_{t}^{\server}, \lambda^\server, K}(\serverqgrad_t)$$ is an approximation of $\servergrad_t$.
\end{itemize}
\item {Workers update:} Every worker updates via the gradient descent step
\begin{equation}\label{eq:distributed_SGD:quantized}
    \bs x_{t+1} = \bs x_t - \del_t \,\workerdgrad_t, 
\end{equation} 
where $\del_t>0$ is the step size in the $t$-th iteration.
\end{itemize}

For a visualization of the communication cost in distributed \ouralgo, see Fig.~\ref{fig:distributed_FQ_SGD}.
Note that this procedure is computationally efficient if the universal sensing basis is a normalized fast Walsh-Hadamard transform. Indeed, both the fast Walsh-Hadamard transform and its inverse can be computed in $O(d\log d)$ operations.
Moreover, at every iteration $t$, each worker has to transmit $2d$ bits to the server, and the server has to transmit $(\log_2(K+1)+1)\cdot d$ bits to every worker.

\begin{figure}
     \centering
     \begin{subfigure}[b]{0.49\textwidth}
         \centering
         \includegraphics[width=\textwidth]{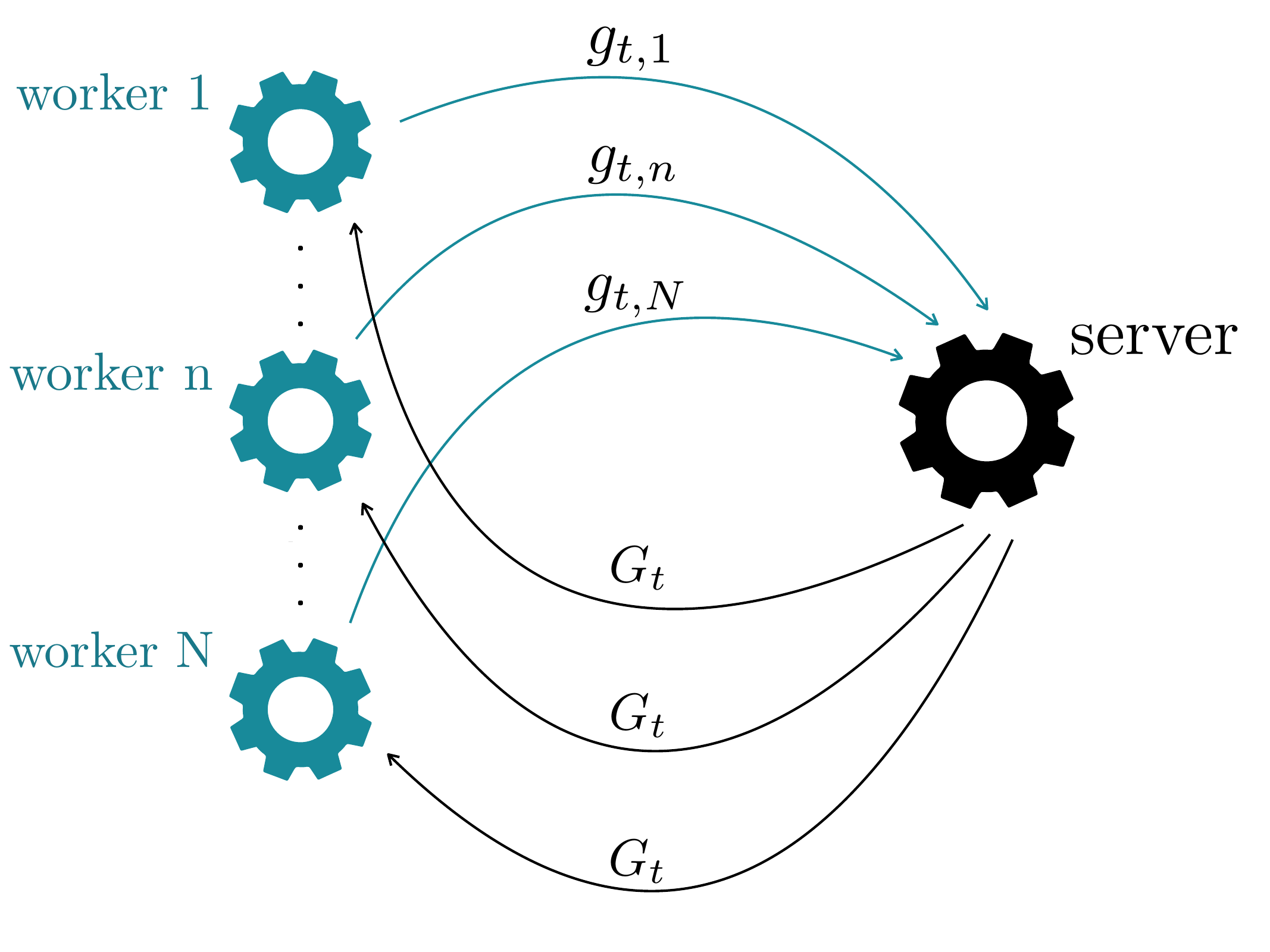}
         \caption{ Communication during the $t$-th round in distributed SGD \\}
         \label{fig:distributed_SGD}
     \end{subfigure}
     \hfill
     \begin{subfigure}[b]{0.49\textwidth}
         \centering
         \includegraphics[width=\textwidth]{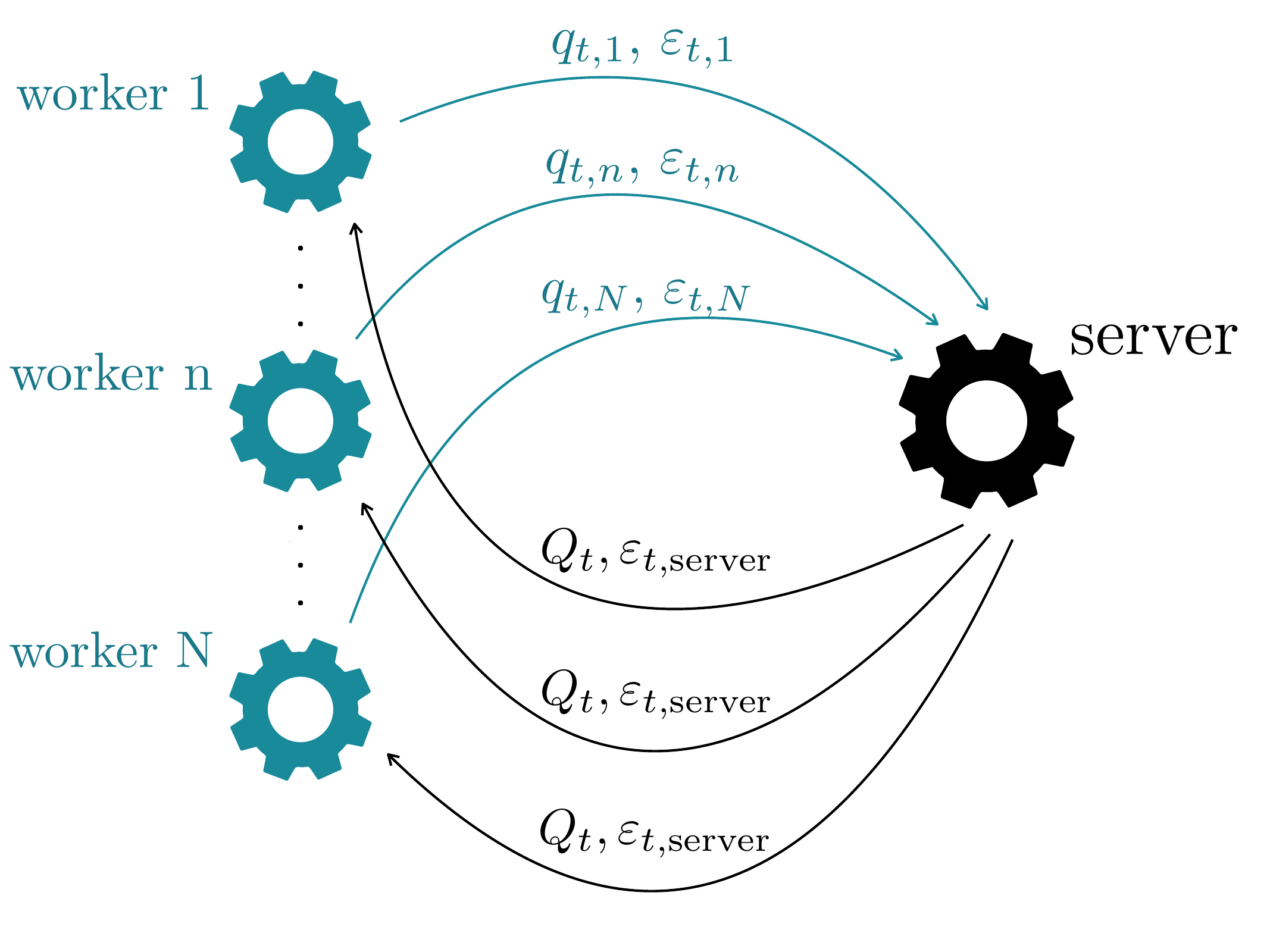}
         \caption{ Quantized communication during the $t$-th round in distributed \ouralgo\\}
         \label{fig:distributed_FQ_SGD}
     \end{subfigure}
     \caption{}
        \label{fig}      
\end{figure}

\section{Convergence Guarantees}
\label{sec:conv_guar}
In order to study the convergence properties of \ouralgo, we first observe that 
this algorithm is an instance of the more general recursion  
\begin{equation}\label{eq:GD:quantized:distributed}
    \bs x_{t+1} = \bs x_t - \del_t \bs v_t, \quad\text{for}\ t=0,\ldots,T,
\end{equation} 
where $\bs v_t$ is of the form 
\begin{equation}\label{eq:GD:quantized:distributed:gradv}
\ts \bs v_t = \Lfct_t^{\server}\Big(\tfrac{1}{N}\sum_{n=1}^N\Lfct_{t,n}\big(\oraclegrad(\bs \xi_{t,n}, \bs x_t)\big)\Big), 
\end{equation}
and where $\oraclegrad(\bs \xi, \cdot):\R^d \to \R^d$ is a stochastic gradient oracle for $f$ according to Def.~\ref{def:stochastic_gradient_oracle}, 
the random vectors $\{\bs \xi_{t,n}\}_{t, n}$ are independent copies of $\bs \xi$, $\{\Lfct_t^{\server}\}_t$ and $\{\Lfct_{t,n}\}_{t,n}$ are some specific (random) functions, independent of $\{\bs \xi_{t,n}\}_{t, n}$.

To see this, let us introduce the following transform map $\Ztransform_{\lambda, K}$, which is directly related to the definition of our quantization scheme.
\begin{definition}\label{def:K_bit_quantizer:flattened}
Given a dithering amplitude $\lambda>0$ and an integer $K\in \N$, let $\dithQ_{\lambda, K}:\R^d\to\R^d$ denote a $K$-averaged dithered one-bit quantizer according to Def.~\ref{def:K_bit_quantizer}. 
Let $\bs \Htransform_{\bs \eps}\in \R^{d\times d}$ be a randomized universal sensing basis, where 
${\bs \eps}\in \{-1,1\}^d$ is independent of $\dithQ_{\lambda, K}$ (\ie independent of the random vectors $\bs \tau_1,\ldots, \bs \tau_K$ which define $\dithQ_{\lambda, K}$). 
Define the map $\Ztransform_{\lambda, K}:\R^d\to \R^d$ by 
\begin{equation*}
\Ztransform_{\lambda, K}(\bs x) := (\bs \Htransform_{\bs \eps})^{-1} \dithQ_{\lambda, K}(\bs \Htransform_{\bs \eps} \bs x)\quad \text{ for all } \bs x\in \R^d.
\end{equation*}
We call $\Ztransform_{\lambda, K}$ a $(\lambda, K)$-transform.
If $K=1$, we simply write $\Ztransform_{\lambda}$ instead of $\Ztransform_{\lambda, 1}$ and call $\Ztransform_\lambda$ a $\lambda$-transform.
\end{definition} 

Using the identity $$\Ztransform_{\lambda, K} = \decoder_{{\bs \eps}, \lambda, K}\circ \encoder_{\bs \eps, \lambda, K},$$
it is easy to see that the update rule of \ouralgo precisely takes the form \eqref{eq:GD:quantized:distributed} with $\{\Lfct_t^{\server}\}_t$ given by independent copies of the random function $\Ztransform_{\lambda^\server, K}$ from  Def.~\ref{def:K_bit_quantizer:flattened} and $\{\Lfct_{t,n}\}_{t,n}$ given by independent copies of the random function $\Ztransform_{\lambda}$ from Def.~\ref{def:K_bit_quantizer:flattened}.
Indeed, the update vector $\workerdgrad_t$ in \eqref{eq:distributed_SGD:quantized} of \ouralgo takes the form: 
\begin{equation}\label{eq:update_vec_FOSGD_decomp}
\workerdgrad_t = \decoder_{{\bs \eps}_{t}^{\server}, \lambda^\server, K}(\encoder_{{\bs \eps}_{t}^{\server}, \lambda^\server, K}(\servergrad_t)) = \Ztransform_{\lambda^\server, K}^{(t)}(\servergrad_t),
\end{equation}
where the vector $\servergrad_t$ is given by \eqref{eq:server_average} and $\Ztransform_{\lambda^\server, K}^{(t)}$ takes the form
\begin{equation*}
\Ztransform_{\lambda^\server, K}^{(t)}(\bs x) = (\bs \Htransform_{\bs \eps^\server_t})^{-1} \dithQ_{\lambda^\server, K}^{(t)}(\bs \Htransform_{\bs \eps^\server_t} \bs x)
\end{equation*}
for 
\begin{equation*}
\dithQ_{\lambda^\server, K}^{(t)}(\bs x) = \tfrac{\lambda^\server}{K}\sum_{i=1}^K\sgn(\bs x+\bs \tau_i^{(t)}),
\end{equation*}
where $\bs \tau_1^{(t)}, \ldots, \bs \tau_K^{(t)}$ are independent and uniformly distributed in $[-\lambda^\server, \lambda^\server]^d$. From \eqref{eq:server_average} and the definition of \ouralgo it follows that the vector $\servergrad_t$ in \eqref{eq:update_vec_FOSGD_decomp} can be decomposed as follows:
\begin{equation}
\ts \servergrad_t = \tfrac{1}{N}\sum_{n=1}^N \tilde{\bs \oraclegrad}_{t,n} = \tfrac{1}{N}\sum_{n=1}^N
\decoder_{{\bs \eps}_{t,n}, \lambda}(\encoder_{{\bs \eps}_{t,n}, \lambda}(\bs \oraclegrad_{t,n})) = \tfrac{1}{N}\sum_{n=1}^N \Ztransform_{\lambda}^{(t,n)}(\bs \oraclegrad_{t,n}),
\end{equation}
where $\bs \oraclegrad_{t,n}$ is an independent copy of the stochastic gradient oracle $\oraclegrad(\bs \xi, \bs x_t)$, i.e., $\bs \oraclegrad_{t,n}= \oraclegrad(\bs \xi_{t,n}, \bs x_t)$ for $\bs \xi_{t,n}$ an independent copy of $\bs \xi$, and $\Ztransform_{\lambda}^{(t,n)}$ takes the form 
\begin{equation*}
\Ztransform_{\lambda}^{(t,n)}(\bs x) = (\bs \Htransform_{\bs \eps_{t,n}})^{-1} \dithQ_{\lambda}^{(t,n)}(\bs \Htransform_{\bs \eps_{t,n}} \bs x)
\end{equation*}
for 
\begin{equation*}
\dithQ_{\lambda}^{(t,n)}(\bs x) = \lambda\sgn(\bs x+\bs \tau^{(t,n)}),
\end{equation*}
where $\bs \tau^{(t,n)}$ are independent and uniformly distributed in $[-\lambda, \lambda]^d$. In total,
\begin{equation}
\workerdgrad_t = \Ztransform_{\lambda^\server, K}^{(t)}(\tfrac{1}{N}\sum_{n=1}^N \Ztransform_{\lambda}^{(t,n)}(\oraclegrad(\bs \xi_{t,n}, \bs x_t))).
\end{equation}

\begin{remark}\label{rem:simple} Notice that $\eu{\Ztransform_{\lambda, K}(\bs x)}\leq \lambda \sqrt{d}$ for all $\bs x\in \R^d$. Indeed, since $\bs \Htransform_{\bs \eps}$ is an orthogonal matrix, we have $\eu{\Ztransform_{\lambda, K}(\bs x)} = \eu{\dithQ_{\lambda, K}(\bs \Htransform_{\bs \eps} \bs x)}$. Therefore, we get the following useful fact by the triangle inequality:
\begin{align*}
\ts \eu{\dithQ_{\lambda, K}(\bs \Htransform_{\bs \eps} \bs x)}\leq \tfrac{\lambda}{K}\sum_{k=1}^K \eu{\sgn(\bs \Htransform_{\bs \eps} \bs x + \bs \tau_k)}=\lambda\sqrt{d}.
\end{align*}
\end{remark}

The following proposition characterizes the convergence of the general algorithm \eqref{eq:GD:quantized:distributed} for general random functions $\{\Lfct_t^{\server}\}_t$ and $\{\Lfct_{t,n}\}_{t,n}$ under the conditions stated above.
 
\begin{proposition}\label{prop:convex:quantized_SGD:distributed} Let $f:\R^d\to \R$ be a convex and differentiable function. Let $\oraclegrad(\bs \xi,\cdot)$ be a stochastic gradient oracle for $f$ with variance bounded by $\sigma^2$.
Let $\bs x_0, \ldots, \bs x_T$ be a sequence of iterates given by \eqref{eq:GD:quantized:distributed} and \eqref{eq:GD:quantized:distributed:gradv}, where $\{\bs \xi_{t,n}\}_{t, n}$ are independent random vectors distributed as $\bs \xi$, and the random functions $\{\Lfct_t^{\server}\}_t$ and $\{\Lfct_{t,n}\}_{t,n}$ are independent of $\{\bs \xi_{t,n}\}_{t, n}$. Set the mean vector $\bar{\bs x}_{T} = c_T\,\sum_{t=0}^T \del_t\, \bs x_t$, with $c_T := (\sum_{t=0}^T \del_t)^{-1}$. Then for any $\bs x_{\ast}\in \R^d$, and $\bs v_t$ defined in \eqref{eq:GD:quantized:distributed:gradv},
\begin{align*}
    &\ts \E[f(\bar{\bs x}_{T})]-f(\bs x_{\ast})\\
    &\ts \leq \tfrac{1}{2} c_T \eu{\bs x_{\ast}-\bs x_0}^2  + c_T \sum_{t=0}^T \del_t^2 \E\big[\eu{\bnabla f(\bs x_t)}^2\big] + \tfrac{2 \sigma^2}{N} \frac{\sum_{t=0}^T \del_t^2}{\sum_{t=0}^T \del_t} \\ 
    &\ts \quad + 2 c_T \sum_{t=0}^T \del_t^2 \E\big[\big\|\bs v_t - \tfrac{1}{N}\sum_{n=1}^N\oraclegrad(\bs \xi_{t,n},\bs x_t)\big\|_2^2\big]\\
    &\ts \quad + c_T \sum_{t=0}^T \del_t \,\big[\big\langle \bs x_t-\bs x_\ast, \tfrac{1}{N}\sum_{n=1}^N \oraclegrad(\bs \xi_{t,n},\bs x_t) - \bs v_t\big\rangle\big].
\end{align*}
\end{proposition}

\begin{proof}[Proof of Prop.~\ref{prop:convex:quantized_SGD:distributed}] 
By convexity of $f$, we get
\begin{multline}
    \ts \E[f(\bar{\bs x}_{T})]-f(\bs x_{\ast})=\E [f(\bar{\bs x}_{T})-f(\bs x_{\ast})]
    \\ \ts \leq \E \big( c_T \big(\sum_{t=0}^T \del_t \, f(\bs x_t)\big) - f(\bs x_{\ast}) \big) = c_T \big(\sum_{t=0}^T \del_t\,\E [f(\bs x_t) - f(\bs x_{\ast})]\big).\label{eq:prop:convex:quantized_SGD:distributed:first_bound} 
\end{multline}
Let us write $\bs \xi_t := (\bs \xi_{t,1}, \ldots, \bs \xi_{t,N})$, define the vector $\bar{\bs g}_t := \tfrac{1}{N}\sum_{n=1}^N \oraclegrad(\bs \xi_{t,n},\bs x_t)$, and let $\E_L$ denote expectation with respect to the randomness of all (random) functions $\{\Lfct_t^{\server}, \Lfct_{t,n}\}_{t,n}$. 
By convexity of $f$ and since $\oraclegrad$ is a stochastic gradient oracle for~$f$, 
\begin{equation*}
    \ts f(\bs x_t) - f(\bs x_\ast) \leq \scp{\bs x_t-\bs x_\ast}{\bnabla f(\bs x_t)}=\langle \bs x_t-\bs x_\ast, \E_{\bs \xi_t} [\bar{\bs g}_t]\rangle.
\end{equation*}
Since $\bs x_t$ does not depend on $\bs \xi_t, \ldots, \bs \xi_T$, we obtain
\begin{align*}
\E_{\bs \xi_0, \ldots, \bs \xi_{T}} [f(\bs x_t) - f(\bs x_{\ast})]
&\ts = \E_{\bs \xi_0, \ldots, \bs \xi_{t-1}} [f(\bs x_t) - f(\bs x_{\ast})]\\
 &\ts \leq \E_{\bs \xi_0, \ldots, \bs \xi_{t-1}}[\langle \bs x_t-\bs x_\ast, \E_{\bs \xi_t} [\bar{\bs g}_t]\rangle]\\
 &\ts = \E_{\bs \xi_0, \ldots, \bs \xi_{t-1}}[\E_{\bs \xi_t} [\langle \bs x_t-\bs x_\ast, \bar{\bs g}_t\rangle]]\\
 &\ts = \E_{\bs \xi_0, \ldots, \bs \xi_{T}}[\langle \bs x_t-\bs x_\ast, \bar{\bs g}_t\rangle].
 \end{align*}
Since $\{\bs \xi_0, \ldots, \bs \xi_T\}$ and $\{\Lfct_t^{\server}, \Lfct_{t,n}\}_{t,n}$ are independent, it follows 
\begin{align*}
\ts \E [f(\bs x_t) - f(\bs x_{\ast})]
\leq \E [\langle \bs x_t-\bs x_\ast, \bar{\bs g}_t\rangle].
 \end{align*}
Adding and subtracting $\bs v_t = \Lfct_t^{\server}\big(\tfrac{1}{N}\sum_{n=1}^N \Lfct_{t,n} ( \oraclegrad(\bs \xi_{t,n},\bs x_t))\big)$, we obtain 
 \begin{equation}\label{eq:prop:convex:quantized_SGD:distributed:single_summand}
\ts \E[f(\bs x_t) - f(\bs x_{\ast})] \leq  \E [\langle \bs x_t-\bs x_\ast, \bs v_t \rangle] + \E [\langle \bs x_t-\bs x_\ast, \bar{\bs g}_t - \bs v_t \rangle ]. 
 \end{equation}

Lemma~\ref{lem:alg_fact} implies 
$$
\ts \sum_{t=0}^T \del_t\, \langle \bs x_t-\bs x_\ast, \bs v_t \rangle \leq \tfrac{1}{2}\eu{\bs x_{\ast}-\bs x_0}^2 +\tfrac{1}{2} \sum_{t=0}^T \del_t^2 \|\bs v_t\|_2^2. 
$$ 

In combination with \eqref{eq:prop:convex:quantized_SGD:distributed:first_bound} and \eqref{eq:prop:convex:quantized_SGD:distributed:single_summand} this shows 
\begin{align*}
    \ts \E[f(\bar{\bs x}_{T})]-f(\bs x_{\ast}) &\leq \tfrac{1}{2}c_T \eu{\bs x_{\ast}-\bs x_0}^2
    \ts + \tfrac{1}{2} c_T \sum_{t=0}^T \del_t^2\, \E[\|\bs v_t\|_2^2]\\ 
    &\quad + c_T \sum_{t=0}^T \delta_t\,\E \big[\big\langle \bs x_t-\bs x_\ast, \bar{\bs g}_t - \bs v_t\big\rangle\big].
\end{align*}
Using the triangle inequality, we obtain  
$$
\ts \E[\|\bs v_t\|_2^2] \leq 2 \E[\|\bnabla f(\bs x_t)\|_2^2] + 2 \E[\|\bs v_t - \bnabla f(\bs x_t)\|_2^2]
$$
and 
$$
\ts \E[\|\bs v_t - \bnabla f(\bs x_t)\|_2^2] \leq 2\E[\|\bs v_t -\bar{\bs g}_t\|_2^2] + 2\E[\|\bar{\bs g}_t-\bnabla f(\bs x_t)\|_2^2].
$$
Since $\oraclegrad(\bs \xi, \cdot)$ has variance bounded by $\sigma^2$
and $\bs \xi_{t,1}, \ldots, \bs \xi_{t,N}$ are independent copies of $\bs \xi$, 
$\bar{\bs g}_t$ is the mean of $N$ independent random vectors, which each have variance bounded by $\sigma^2$.
Using that $\E_{\bs \xi_t}(\bar{\bs g}_t)=\bnabla f(\bs x_t)$, it follows from Corollary~\ref{coro:var_decrease} that
\begin{equation*}
\ts \E_{\bs \xi_t}  \big[\big\|\bar{\bs g}_t-\bnabla f(\bs x_t)\big\|_2^2\big]\leq \tfrac{\sigma^2}{N}. 
\end{equation*}
By independence of the random vectors $\bs \xi_0, \ldots, \bs \xi_T$, and the independence of $\{\bs \xi_0, \ldots, \bs \xi_T\}$ and $\{\Lfct_t^{\server}, \Lfct_{t,n}\}_{t,n}$, this implies
\begin{equation*}
\ts \E  \big[\big\|\bar{\bs g}_t-\bnabla f(\bs x_t)\big\|_2^2\big]\leq \tfrac{\sigma^2}{N}.
\end{equation*}
Putting everything together, we obtain the result.
\end{proof}

The following immediate corollary, which concerns distributed SGD \emph{without} compressed communication, allows us to assess the bound in Prop.~\ref{prop:convex:quantized_SGD:distributed}.
\begin{corollary} 
\label{cor:uncompress-main-result}
Let $f:\R^d\to \R$ be convex and differentiable. Let $\oraclegrad(\bs \xi,\cdot)$ be a stochastic gradient oracle for $f$ with variance bounded by $\sigma^2$.
Let $\bs x_0, \ldots, \bs x_T$ be a sequence of iterates given by 
\begin{equation}
\bs x_{t+1} = \bs x_t - \del_t \cdot \tfrac{1}{N}\sum_{n=1}^N\oraclegrad(\bs \xi_{t,n}, \bs x_t),
\end{equation}
where $\{\bs \xi_{t,n}\}_{t, n}$ are independent random vectors distributed as $\bs \xi$. Set the mean vector $\bar{\bs x}_{T} = c_T\,\sum_{t=0}^T \del_t\, \bs x_t$, with $c_T := (\sum_{t=0}^T \del_t)^{-1}$. Then for any $\bs x_{\ast}\in \R^d$,
\begin{equation}
    \label{eq:uncompress-main-result}
    \ts \E[f(\bar{\bs x}_{T})]-f(\bs x_{\ast}) \leq \cl C_{f,T,\oraclegrad}(\bs x_{\ast}, \bs x_0)
\end{equation} 
with 
$$
\ts \cl C_{f,T,\oraclegrad}(\bs x_{\ast}, \bs x_0)
 :=\tfrac{1}{2} c_T \eu{\bs x_{\ast}-\bs x_0}^2  + c_T \sum_{t=0}^T \del_t^2 \E\big[\eu{\bnabla f(\bs x_t)}^2\big] + \tfrac{2\sigma^2}{N} \frac{\sum_{t=0}^T \del_t^2}{\sum_{t=0}^T \del_t}.
$$
\end{corollary}

From this corollary, we thus observe that the bound on $\E[f(\bar{\bs x}_{T})]-f(\bs x_{\ast})$ provided by Prop.~\ref{prop:convex:quantized_SGD:distributed} amounts to adding two terms to the bound $\cl C_{f,T,\oraclegrad}(\bs x_{\ast}, \bs x_0)$. These encode the impact of the compression scheme through the interplay of the general functions $\Lfct_t^{\server}$ and $\Lfct_{t,n}$. In the context of \ouralgo, the specific forms of these two functions allows us to derive the following theorem.

\begin{theorem}\label{thm:Hadamard:all_compressed} 
Let $f:\R^d\to \R$ be convex and differentiable, and $\oraclegrad(\bs \xi,\cdot)$ be a stochastic gradient oracle for $f$ such that, for $B > 0$, 
\begin{equation*}
\eu{\oraclegrad(\bs \xi, \bs x)}\leq B\quad \text{ for all } \bs x\in \R^d.
\end{equation*}
There exists an absolute constant $C\geq 2$ such that the following holds.
For parameters $\alpha, \alpha^\server\geq C$, set the dithering amplitudes $\lambda$ and $\lambda^\server$ to
$$
\lambda = \alpha B\sqrt{\tfrac{\log d}{d}},\ \text{and}\ \lambda^\server = \alpha^\server \lambda \sqrt{\log d}.
$$
Let $\bs x_0, \ldots, \bs x_T$ be the sequence of iterates given by \eqref{eq:GD:quantized:distributed} and \eqref{eq:GD:quantized:distributed:gradv} with constant learning rate $\del_t=\del$, 
$\{\bs \xi_{t,n}\}_{t, n}$ independent copies of the random variable $\bs \xi$, 
$\Lfct_{t,n}$ independent copies of the $\lambda$-transform $\Ztransform_{\lambda}:\R^d\to \R^d$ from  Def.~\ref{def:K_bit_quantizer:flattened}, and 
$\Lfct_t^{\server}$ independent copies of the $(\lambda^\server, K)$-transform $\Ztransform_{\lambda^\server, K}:\R^d\to \R^d$ from Def.~\ref{def:K_bit_quantizer:flattened}. We assume that all involved random variables are independent. Let us set $\bar{\bs x}_{T}=\tfrac{1}{T+1}\sum_{t=0}^T\bs x_t$ and define the average distance 
$$
\ts \eta^*_T(\bs x_{\ast}) := \E \Big[ \tfrac{1}{T+1}\sum_{t=0}^T\eu{\bs x_t-\bs x_\ast}\Big].
$$

\noindent Then for any $\bs x_{\ast}\in \R^d$,
\begin{align*}
    &\E[f(\bar{\bs x}_{T})]-f(\bs x_{\ast})\\ \nonumber
    &\leq \cl C_{f,T,\oraclegrad}(\bs x_{\ast}, \bs x_0) + \alpha^2B^2\,\del\,\Big( \tfrac{8\log d}{N}
    + \tfrac{4(\alpha^\server)^2 \log^2 d}{K} \Big)\\ \nonumber
    &\quad  + 
 32 B^2 \,\del\,\big(N \exp(-\tfrac{1}{8}\alpha^2\log d) + \alpha^2  \log (d) \exp(-\tfrac{1}{8} (\alpha^\server)^2 \log d)\big)\\ \nonumber
    &\quad +  4B\,\eta^*_T(\bs x_{\ast}) \,\big( N \exp(-\tfrac{1}{8}\alpha^2\log d) + \alpha  \sqrt{\log d} \exp(-\tfrac{1}{8}(\alpha^\server)^2\log d) \big).
\end{align*}
Here, 
$$
\ts \cl C_{f,T,\oraclegrad}(\bs x_{\ast}, \bs x_0)
 =\tfrac{1}{2(T+1)\del} \eu{\bs x_{\ast}-\bs x_0}^2  + \tfrac{\del}{T+1} \sum_{t=0}^T \E\big[\eu{\bnabla f(\bs x_t)}^2\big] + \tfrac{2\sigma^2\del}{N}.
$$
\end{theorem}

\begin{proof}
By Prop.~\ref{prop:convex:quantized_SGD:distributed}, with $\bs v_t = \Lfct_t^{\server}\big(\tfrac{1}{N}\sum_{n=1}^N \Lfct_{t,n} (\oraclegrad(\bs \xi_{t,n},\bs x_t))\big)$ and the mean gradient oracle $\bar{\bs g}_t = \tfrac{1}{N}\sum_{n=1}^N \oraclegrad(\bs \xi_{t,n},\bs x_t)$, we get 
\begin{multline}\label{eq:thm:Hadamard:all_compressed:main}
    \ts \E[f(\bar{\bs x}_{T})]-f(\bs x_{\ast}) \leq \cl C_{f,T,\oraclegrad}(\bs x_{\ast}, \bs x_0)
    \ts  + \tfrac{2\del}{T+1} \sum_{t=0}^T\E[\|\bs v_t - \bar{\bs g}_t\|_2^2] + \tfrac{1}{T+1}\sum_{t=0}^T \E [\langle \bs x_t-\bs x_\ast, \bar{\bs g}_t - \bs v_t \rangle].
\end{multline}
Define the vector 
\begin{equation*}
\ts \bs w_t := \tfrac{1}{N}\sum_{n=1}^N \Lfct_{t,n} (\oraclegrad(\bs \xi_{t,n},\bs x_t))
\end{equation*}
so that $\bs v_t = \Lfct_t^{\server}(\bs w_t)$. We proceed by estimating 
\begin{equation}\label{eq:thm:Hadamard:all_compressed:first}
\ts \E[\|\bs v_t - \bar{\bs g}_t\|_2^2] \leq 2\E[\|\bs v_t - \bs w_t\|_2^2] + 2\E[\|\bs w_t - \tfrac{1}{N}\sum_{n=1}^N\oraclegrad(\bs \xi_{t,n},\bs x_t)\|_2^2]
\end{equation}
and writing 
\begin{equation}
\label{eq:thm:Hadamard:all_compressed:second}
\ts \E [\langle \bs x_t-\bs x_\ast, \bar{\bs g}_t - \bs v_t \rangle] = \E [\langle \bs x_t-\bs x_\ast, \bar{\bs g}_t - \bs w_t \rangle] + \E [\langle \bs x_t-\bs x_\ast, \bs w_t - \bs v_t \rangle].
\end{equation}
It remains to estimate the terms on the right hand side of \eqref{eq:thm:Hadamard:all_compressed:first} and 
\eqref{eq:thm:Hadamard:all_compressed:second}.
For $t\in \{0,\ldots, T\}$ and $n\in \{1, \ldots, N\}$, let  
\begin{equation*}
\Lfct_{t,n}(\bs x)= \lambda\, \bs \Htransform_{{\bs \eps}_{t,n}}^{-1} \sgn(\bs \Htransform_{{\bs \eps}_{t,n}} \bs x+\bs \tau_{t,n})
\end{equation*}
and 
\begin{equation*}
\ts \Lfct_t^{\server}(\bs x) = \tfrac{\lambda^\server}{K}\, \bs \Htransform_{{\bs \eps}_t^\server}^{-1} \sum_{k=1}^K\sgn(\bs \Htransform_{{\bs \eps}_t^\server} \bs x+\bs \tau_{t,k}^\server)
\end{equation*}
for independent and uniformly distributed random vectors
${\bs \eps}_{t,n}\in \{-1,1\}^d$, $\bs \tau_{t,n}\in [-\lambda, \lambda]^d$, ${\bs \eps}^\server_t\in \{-1,1\}^d, \bs \tau_{t, k}^\server\in [-\lambda^\server, \lambda^\server]^d$.

We define the sets of vectors $\bs \xi_t := \{\bs \xi_{t,1}, \ldots, \bs \xi_{t,N}\}$, ${\bs \eps}_t := \{\bs \eps_{t,1}, \ldots, \bs \eps_{t,N}\}$, $\bs \tau_t := \{\bs \tau_{t,1}, \ldots, \bs \tau_{t,N}\}$, $\bs \tau_t^\server := \{\bs \tau_{t, 1}^\server, \ldots, \bs \tau_{t, K}^\server\}$, as well as the global set
\begin{equation*}
\cl X_t := \{\bs \xi_0, \ldots, \bs \xi_{t}, \bs \eps_0,\ldots, \bs \eps_{t}, \bs \tau_0,\ldots, \bs \tau_{t}, \bs \eps_0^\server, \ldots, \bs \eps_t^\server,  \bs \tau_{0}^\server, \ldots, \bs \tau_t^\server\}. 
\end{equation*}
Observe that $\bs x_t$ is completely determined by $\cl X_{t-1}$.
Conditioned on $\cl X_{t-1}$ and $\bs \xi_t$, we have $\max_{n\in [N]}\eu{\oraclegrad(\bs \xi_{t,n},\bs x_t)}\leq B$. Therefore, 
Corollary~\ref{lem:Hadamard:average} and Lemma~\ref{lem:K:rad} imply 
$$
\ts \E_{\bs \eps_t, \bs \tau_t}\big[\big\|\bs w_t - \bar{\bs g}_t\big\|_2^2\big] \leq \tfrac{2\alpha^2B^2\log d}{N} + 8N B^2 \exp(-\tfrac{1}{8}\alpha^2\log d).
$$
Taking expectation with respect to $\cl X_{t-1}$ and $\bs \xi_t$, this yields
\begin{equation}
\label{eq:thm:Hadamard:all_compressed:first:1}
\ts\E[\|\bs w_t - \bar{\bs g}_t\|_2^2] \leq \tfrac{2\alpha^2B^2\log d}{N} + 8N B^2 \exp(-\tfrac{1}{8}\alpha^2\log d).
\end{equation}
For $n\in [N]$ define the events  
\begin{equation*}
\mathcal{A}_n:=\{\|\bs \Htransform_{\bs \eps_{t,n}} \oraclegrad(\bs \xi_{t,n},\bs x_t)\|_\infty\leq \lambda\}
\end{equation*}
and $\mathcal{A}:=\cap_{n=1}^N \mathcal{A}_n$.
By Lemma~\ref{lem:Hadamard} we have for any $n\in [N]$, 
\begin{equation*}
\Pb_{\bs \eps_{t,n}}\big(\|\bs \Htransform_{\bs \eps_{t,n}} \oraclegrad(\bs \xi_{t,n},\bs x_t)\|_\infty> \alpha\sqrt{\tfrac{\log d}{d}}\eu{\oraclegrad(\bs \xi_{t,n},\bs x_t)}\big)
\leq 2\exp(-\tfrac{1}{4}\alpha^2\log d).
\end{equation*}
Since $\lambda \geq \alpha\sqrt{\tfrac{\log d}{d}}\eu{\oraclegrad(\bs \xi_{t,n},\bs x_t)}$ for all $n\in [N]$, it follows 
\begin{align*}
\Pb_{\bs \eps_t}(\mathcal{A}^C)&\ts \leq \sum_{n=1}^N \Pb_{\bs \eps_{t,n}}(\mathcal{A}_{n}^C)\\
&\ts \leq \sum_{n=1}^N \Pb_{\bs \eps_{t,n}}\big(\|\bs \Htransform_{\eps_{t,n}} \oraclegrad(\bs \xi_{t,n},\bs x_t)\|_\infty> \alpha\sqrt{\tfrac{\log d}{d}}\eu{\oraclegrad(\bs \xi_{t,n},\bs x_t)}\big)\\
&\leq 2N\exp(-\tfrac{1}{4}\alpha^2\log d).
\end{align*}
Lemma~\ref{lem:K_bit_quantizer} implies that on the event $\mathcal{A}$, $\E_{\bs \tau_{t,n}} \big[\Lfct_{t,n} (\oraclegrad(\bs \xi_{t,n},\bs x_t))\big]=\oraclegrad(\bs \xi_{t,n},\bs x_t)$ for all $n\in[N]$, so that, on this event, $\bb E_{\bs \tau_{t}} \bs w_t = \bar{\bs g}_t$. Therefore,
$$
\ts \E [\langle \bs x_t-\bs x_\ast, \bar{\bs g}_t - \bs w_t\rangle]\\
 =\E \big[1_{\mathcal{A}^C} \cdot \langle \bs x_t-\bs x_\ast, \bar{\bs g}_t - \bs w_t \rangle\big]. 
$$
Using Cauchy-Schwarz and the trivial estimate $\sup_{\bs x\in\R^d}\|\Lfct_{t,n}(\bs x)\|_2\leq \lambda \sqrt{d}$, we obtain 
\begin{align}\label{eq:thm:Hadamard:all_compressed:second:1}
\ts \E [\langle \bs x_t-\bs x_\ast, \bar{\bs g}_t - \bs w_t\rangle]&\leq (B+\lambda \sqrt{d})\cdot \E \big[1_{\mathcal{A}^C}\eu{\bs x_t-\bs x_\ast}\big]\\ \nonumber
&=(B+\lambda \sqrt{d})\cdot \E_{\cl X_{t-1}} \big[\eu{\bs x_t-\bs x_\ast}\E_{\bs \xi_t}\E_{\bs \eps_t}[1_{\mathcal{A}^C}]\big]\\ \nonumber
&\leq (B+\lambda \sqrt{d})\cdot 2N\exp(-\tfrac{1}{4}\alpha^2\log d) \cdot \E\big[\eu{\bs x_t-\bs x_\ast}\big].
\end{align}
Clearly, $B\leq \lambda \sqrt{d}$ and 
\begin{align*}
\lambda \sqrt{d} \exp(-\tfrac{1}{4}\alpha^2\log d)&= B \alpha \sqrt{\log d}\exp(-\tfrac{1}{4}\alpha^2\log d)\\
&\leq B \exp(-\tfrac{1}{8}\alpha^2\log d)
\end{align*}
for all $\alpha\geq C$, where $C>0$ denotes an absolute constant that is chosen large enough. 

Observe that $\eu{\bs w_t}\leq \lambda \sqrt{d}$. Lemma~\ref{lem:K:rad} implies that 
\begin{equation*}
\E_{\bs \eps_t^\server, \bs \tau_t^\server}\big[\big\|\bs v_t - \bs w_t\big\|_2^2\big]
\leq \tfrac{(\alpha^\server)^2\cdot (\lambda \sqrt{d})^2 \cdot \log d}{K}
+ 8(\lambda \sqrt{d})^2 \exp(-\tfrac{1}{8} (\alpha^\server)^2 \log d).
\end{equation*}
Using $\lambda =\alpha B\sqrt{\tfrac{\log d}{d}}$ and taking expectations with respect to $\cl X_{t-1}, \bs \xi_t, \bs \eps_t, \bs \tau_t$, we obtain 
\begin{equation}\label{eq:thm:Hadamard:all_compressed:first:2}
\E\big[\big\|\bs v_t - \bs w_t\big\|_2^2\big] \leq \tfrac{(\alpha^\server)^2\cdot \alpha^2 B^2 \log^2 d}{K} + 8\alpha^2 B^2 \log( d )\exp(-\tfrac{1}{8} (\alpha^\server)^2 \log d).
\end{equation}
Define the event 
\begin{equation*}
\mathcal{A}^\server := \{\|\bs \Htransform_{\bs \eps_t^\server}\bs w_t\|_\infty\leq \lambda^\server\}.
\end{equation*}
By Lemma~\ref{lem:Hadamard},  
\begin{equation*}
\Pb_{\bs \eps_t^\server}\big(\|\bs \Htransform_{\bs \eps_t^\server}\bs w_t\|_\infty> \alpha^\server\sqrt{\tfrac{\log d}{d}}\eu{\bs w_t}\big)
\leq 2\exp(-\tfrac{1}{4}(\alpha^\server)^2\log d).
\end{equation*}
Since $\lambda^\server\geq \alpha^\server\sqrt{\tfrac{\log d}{d}}\eu{\bs w_t}$, it follows 
\begin{equation*}
\Pb_{\bs \eps_t^\server}((\mathcal{A}^\server)^C)\leq 2\exp(-\tfrac{1}{4}(\alpha^\server)^2\log d).
\end{equation*}
Lemma~\ref{lem:K_bit_quantizer} implies that on the event $\mathcal{A}^\server$, 
\begin{equation*}
\E_{\bs \tau_t^\server}[\bs v_t] = \E_{\bs \tau_t^\server} \big[\Lfct_t^{\server}(\bs w_t)\big] = \bs w_t.
\end{equation*}
Therefore, 
\begin{equation*}
\E \big[\big\langle \bs x_t-\bs x_\ast, \bs w_t - \bs v_t\big\rangle\big]
= \E \big[1_{(\mathcal{A}^\server)^C}\cdot\big\langle \bs x_t-\bs x_\ast, \bs w_t - \bs v_t\big\rangle\big].
\end{equation*}
Using Cauchy-Schwarz and the trivial bound $\eu{\Lfct_t^{\server}(\bs w_t)}\leq \lambda^\server \sqrt{d}$ (see Remark~\ref{rem:simple}), we obtain 
\begin{align}\label{eq:thm:Hadamard:all_compressed:second:2}
&\E \big[\big\langle \bs x_t-\bs x_\ast, \bs w_t - \bs v_t\big\rangle\big]\\ \nonumber
&\leq (\lambda \sqrt{d} + \lambda^\server \sqrt{d})\cdot \E \big[ 1_{(\mathcal{A}^\server)^C} \eu{\bs x_t-\bs x_\ast}\big]\\ \nonumber
&\leq (\lambda \sqrt{d} + \lambda^\server \sqrt{d})\cdot \E_{\cl X_{t-1}} \big[ \eu{\bs x_t-\bs x_\ast} \E_{\bs \xi_t, \bs \eps_t, \bs \tau_t} \E_{\bs \eps_t^\server}[1_{(\mathcal{A}^\server)^C}] \big]\\ \nonumber
&\leq 2\sqrt{d}\cdot (\lambda + \lambda^\server)\cdot \exp(-\tfrac{1}{4}(\alpha^\server)^2\log d) \cdot\E \big[ \eu{\bs x_t-\bs x_\ast}\big].
\end{align}
Finally, observe that $\lambda \leq \lambda^\server $ and 
\begin{align*}
\lambda^\server \sqrt{d}\cdot \exp(-\tfrac{1}{4}(\alpha^\server)^2\log d) &= (\lambda \sqrt{d})\cdot \alpha^\server \sqrt{\log d} \cdot \exp(-\tfrac{1}{4}(\alpha^\server)^2\log d)\\
&\leq \lambda \sqrt{d} \cdot \exp(-\tfrac{1}{8}(\alpha^\server)^2\log d)
\end{align*}
for all $\alpha^\server\geq C$, where $C>0$ denotes an absolute constant that is chosen large enough.
The result now follows by combining \eqref{eq:thm:Hadamard:all_compressed:main}, \eqref{eq:thm:Hadamard:all_compressed:first}, \eqref{eq:thm:Hadamard:all_compressed:second}, \eqref{eq:thm:Hadamard:all_compressed:first:1}, 
\eqref{eq:thm:Hadamard:all_compressed:second:1},
\eqref{eq:thm:Hadamard:all_compressed:first:2} and \eqref{eq:thm:Hadamard:all_compressed:second:2}.
\end{proof}

{\section{Non-convex Convergence Rate}
\label{sec:non-convex}

We now show that \ouralgo can be used to find critical points of a non-convex, smooth function $f$. This amounts to showing that the gradient of $f$ on the iterates $\bs x_t$ of \ouralgo vanishes when $t$ increases. We here follow the procedure explained in \cite{BW18,BZ19} and adapt it to our specific gradient compression scheme. 

We thus assume that $f$ is $L$-smooth for some constant $L>0$, that is, for any $\bs p, \bs q \in \bb R^d$,
\begin{equation}
\label{eq:def-L-smooth}
\ts f(\bs q) \leq f(\bs p) + \scp[\big]{\nabla f(\bs p)}{\bs q - \bs p} + \frac{1}{2} L \|\bs q - \bs p\|^2.
\end{equation}

Let us first recap the main concepts involved in our algorithm. It is thus an instance of the more general recursion  
$$
\bs x_{t+1} = \bs x_t - \del_t \bs v_t, \quad\text{for}\ t=0,\ldots,T,\ \bs x_0 \in \bb R^d,
$$
where $\bs v_t$ is of the form 
$$
\ts \bs v_t = \Lfct_t^{\server}\Big(\tfrac{1}{N}\sum_{n=1}^N\Lfct_{t,n}\big(\oraclegrad(\bs \xi_{t,n}, \bs x_t)\big)\Big), 
$$
and where $\oraclegrad(\bs \xi, \cdot):\R^d \to \R^d$ is a stochastic gradient oracle for $f$ according to Def.~\ref{def:stochastic_gradient_oracle}, 
the random vectors $\{\bs \xi_{t,n}\}_{t, n}$ are independent copies of $\bs \xi$, $\{\Lfct_t^{\server}\}_t$ and $\{\Lfct_{t,n}\}_{t,n}$ are some specific (random) functions, independent of $\{\bs \xi_{t,n}\}_{t, n}$.

More specifically, after simplification, the functions $\Lfct_t^{\server}$ and $\Lfct_{t,n}$ are defined~by
\begin{align*}
    \Lfct_t^{\server}(\bs p)&\ts = \tfrac{\lambda^\server}{K}\sum_{k=1}^K \bs \Htransform_{\bs \eps_t^\server}^{-1} \sgn(\bs \Htransform_{\bs \eps_t^\server} \bs p + \bs \tau^\server_{t,k}),&&\bs p \in \bb R^d,\\
    \Lfct_{t,n}(\bs p)&\ts =  \lambda \bs \Htransform_{\bs \eps_{t,n}}^{-1} \sgn(\bs \Htransform_{\bs \eps_{t,n}} \bs p + \bs \tau_{t,n}),&&\bs p \in \bb R^d,\\
    \ts \tfrac{1}{N}\sum_{n=1}^N \Lfct_{t,n}(\bs p_n)&\ts = \tfrac{\lambda}{N}\sum_{n=1}^N \bs \Htransform_{\bs \eps_{t,n}}^{-1} \sgn(\bs \Htransform_{\bs \eps_{t,n}} \bs p_n + \bs \tau_{t,n}),&&\bs p_n \in \bb R^d,
\end{align*}
with the random vectors $\bs \tau^\server_{t,k} \sim_{\rm iid} \cl U(\lambda^\server [-1,1]^d)$, $\bs \tau_{t,n} \sim_{\rm iid} \cl U(\lambda [-1,1]^d)$, $\bs \eps_t^\server \sim_{\rm iid} \cl U(\{\pm 1\}^d)$, and $\bs \eps_{t,n} \sim_{\rm iid} \cl U(\{\pm 1\}^d)$. We notice the similarity of shape between the first and the third quantities in the above expressions.  

Moreover, defining the \emph{clipping} function $\clip_\lambda(x)$ equal to $x/\lambda$ if $|x| \leq \lambda$ and $\sgn(x)$ otherwise (to be also applied componentwise onto vectors), you have shown the following properties, for $\bs p, \bs p_n \in \bb R^{d}$,
\begin{align*}
    \bb E_{\bs \tau^\server_{t,:}}\Lfct_t^{\server}(\bs p)&\ts = \lambda^\server \bs \Htransform_{\bs \eps_t^\server}^{-1} \clip_{\lambda^\server}\big(\bs \Htransform_{\bs \eps_t^\server} \bs p),\\
    \bb E_{\bs \tau_{t,n}} \Lfct_{t,n}(\bs p)&\ts =  \lambda \bs \Htransform_{\bs \eps_{t,n}}^{-1} \clip_\lambda(\bs \Htransform_{\bs \eps_{t,n}} \bs p),\\
    \bb E_{\bs \tau_{t,:}} \ts \tfrac{1}{N}\sum_{n=1}^N \Lfct_{t,n}(\bs p_n)&\ts = \tfrac{\lambda}{N}\sum_{n=1}^N \bs \Htransform_{\bs \eps_{t,n}}^{-1} \clip_\lambda(\bs \Htransform_{\bs \eps_{t,n}} \bs p_n).
\end{align*}
Therefore, given $t$, if $\|\bs \Htransform_{\bs \eps_t^\server} \bs p\|_\infty \leq \lambda^\server$, and  for all $\|\bs \Htransform_{\bs \eps_{t,n}} \bs p_n\|_\infty \leq \lambda$ for all $n$, we have 
$$
\ts \bb E_{\bs \tau^\server_{t,:}}\Lfct_t^{\server}(\bs p) \ts = \bb E_{\bs \tau_{t,n}} \Lfct_{t,n}(\bs p) = \bs p, \quad \bb E_{\bs \tau_{t,:}} \tfrac{1}{N}\sum_{n=1}^N \Lfct_{t,n}(\bs p_n) = \tfrac{1}{N}\sum_{n=1}^N \bs p_n.
$$
Consequently, using the same tools that you develop later around \eqref{eq:thm:Hadamard:all_compressed:second:1} and \eqref{eq:thm:Hadamard:all_compressed:second:2} (with events similar to $\cl A$ and $\cl A'$), one could show, that by taking total expectation, that if there are some $B >0$ such that if $\|\bs p\|, \|\bs p_n\| \leq B$ for all $n$,  
\begin{align}
\label{eq:lpd-server}
\big|\scp[\big]{\bs u}{\bb E_{\bs \epsilon^\server_{t}, \bs \tau^\server_{t,:}} \Lfct_t^{\server}(\bs p) - \bs p}\big|&\leq \rho^\server \|\bs u\|,\\
\label{eq:lpd-worker}
\big|\scp[\big]{\bs u}{\bb E_{\bs \epsilon_{t,n}, \bs \tau_{t,n}} \Lfct_t^{\server}(\bs p) - \bs p}\big|&\leq \rho \|\bs u\|,
\end{align}
for any $\bs u \in \bb R^d$, and some distortions 
\begin{align*}
\rho^\server&\ts \leq 2 \sqrt d (\lambda + \lambda^\server) \exp(- \frac{(\alpha^\server)^2}{4} \log d),\\
\rho&\ts \leq 2 N (B + \lambda \sqrt d) \exp(- \frac{\alpha^2}{4} \log d),     
\end{align*} 
provided that, up to some log factors, $\lambda^\server \geq C \alpha^\server B$ and $\lambda \geq C \alpha B$ (there are certainly a few bolts to screw here to get the bounds right).

With properties \eqref{eq:lpd-server} and \eqref{eq:lpd-worker}, we can show the following proposition. 
\begin{proposition}
If $f$ is $L$-smooth, $\max_n \|\bs g_{t,n}\| \leq B$, $\alpha^\server, \alpha = \Omega((\log T)^q)$, then \ouralgo with step sizes $\delta_t$ reaches
\begin{align*}
&\ts c_T \sum_{t=0}^{T-1} \del_t \bb E \|\nabla f(\bs x_t)\|^2\\
&\ts \leq c_T \big(f(\bs x_0) - f(\bs x_*)\big) + (\rho + \rho^\server) B + \frac{1}{2} L (\lambda^\server)^2 d c_T \sum_{t=0}^{T-1} \del_t^2 ,
\end{align*}
with $(\rho + \rho^\server) = O(T^{-q})$, and $c_T = (\sum_{t=0}^{T-1} \del_t)^{-1}$. In particular, for $\delta_t = 1/\sqrt T$, we get 
\begin{align*}
&\ts \frac{1}{T} \sum_{t=0}^{T-1} \bb E \|\nabla f(\bs x_t)\|^2\\
&\ts \leq \frac{1}{\sqrt T} \big(f(\bs x_0) - f(\bs x_*)\big) + (\rho + \rho^\server) B + \frac{1}{2 \sqrt T} L (\lambda^\server)^2 d.
\end{align*}

\end{proposition}
\begin{proof}
From the $L$-smoothness of $f$, we have 
\begin{align}
    f(\bs x_{t+1}) - f(\bs x_t)&\ts \leq \scp[\big]{\nabla f(\bs x_t)}{\bs x_{t+1} - \bs x_t} + \frac{1}{2} L \|\bs x_{t+1} - \bs x_t\|^2\nonumber\\
    &\ts = - \del_t \scp[\big]{\nabla f(\bs x_t)}{\bs v_t} + \frac{1}{2} L \del_t^2 \|\bs v_t\|^2. \label{eq:using-L-smooth}
\end{align}
Moreover, using \eqref{eq:lpd-server} and \eqref{eq:lpd-worker}, and writing $\bs g_{n,t} := \oraclegrad(\bs \xi_{t,n}, \bs x_t)$, $\bs \theta_t := \{\bs \epsilon^\server_{t}, \bs \tau^\server_{t,:}, \bs \epsilon_{t,:}, \bs \tau_{t,:}\}$, $\bs w_t := \tfrac{1}{N}\sum_{n=1}^N\Lfct_{t,n}(\bs g_{n,t})$, and $\bar g_t := \tfrac{1}{N}\sum_{n=1}^N \bs g_{n,t}$, we find, for any $\bs u$ (possibly random but independent of $\bs \theta_t$,
\begin{align*}
\ts \scp[\big]{\bs u}{\bb E_{\bs \theta_t} \bs v_t}&\ts =\scp[\big]{\bs u}{\bb E_{\bs \theta_t} \Lfct_t^{\server}(\bs w_t)} = \scp[\big]{\bs u}{\bb E_{\bs \theta_t} \bs w_t} + \scp[\big]{\bs u}{\bb E_{\bs \theta_t}[\Lfct_t^{\server}(\bs w_t) - \bs w_t]}\\
&\ts \geq \scp[\big]{\bs u}{\bb E_{\bs \theta_t} \bs w_t} - \rho^\server \|\bs u\| \\
&\ts = \tfrac{1}{N}\sum_{n=1}^N \scp[\big]{\bs u}{\bb E_{\bs \theta_t} \Lfct_{t,n}(\bs g_{n,t})} - \rho^\server \|\bs u\| \\
&\ts = \scp[\big]{\bs u}{\bar{\bs g}_{t}} + \tfrac{1}{N}\sum_{n=1}^N \scp[\big]{\bs u}{\bb E_{\bs \theta_t} \Lfct_{t,n}(\bs g_{n,t}) - \bs g_{n,t}} - \rho^\server \|\bs u\| \\
&\ts \geq \scp[\big]{\bs u}{\bar{\bs g}_{t}} - (\rho + \rho^\server) \|\bs u\|.
\end{align*}

Therefore, taking $\bs u = \nabla f(\bs x_t)$, $\scp[\big]{\nabla f(\bs x_t)}{\bb E_{\bs \theta_t} \bs v_t} \geq \scp[\big]{\nabla f(\bs x_t)}{\bar{\bs g}_{t}} - (\rho + \rho^\server) \|\bs \nabla f(\bs x_t)\|$, a total expectation applied to  \eqref{eq:using-L-smooth} (paying attention to the factoring of expectations) gives
\begin{align*}
    \bb E f(\bs x_{t+1}) - \bb E f(\bs x_t)&\ts \leq - \del_t \bb E \|\nabla f(\bs x_t)\|^2 + \del_t (\rho + \rho^\server) \bb E\|\nabla f(\bs x_t)\| + \frac{1}{2} L \del_t^2 \|\bs v_t\|^2.
\end{align*}

This means that, for the minimum $f(\bs x_*)$,
\begin{align*}
    f(\bs x_0) - f(\bs x_*)
    &\ts \geq f(\bs x_0) - f(\bs x_T)\\
    &\ts \geq \sum_{t=0}^{T-1} \bb E [f(\bs x_t) - f(\bs x_{t+1})]\\
    &\ts \geq \sum_{t=0}^{T-1} \del_t \bb E \|\nabla f(\bs x_t)\|^2\\
    &\ts \quad - (\rho + \rho^\server) \sum_{t=0}^{T-1} \del_t \bb E\|\nabla f(\bs x_t)\| - \frac{1}{2} L \sum_{t=0}^{T-1} \del_t^2 \|\bs v_t\|^2.
\end{align*}
Equivalently,  
\begin{align*}
&\ts c_T \sum_{t=0}^{T-1} \del_t \bb E \|\nabla f(\bs x_t)\|^2\\
&\ts \leq c_T \big(f(\bs x_0) - f(\bs x_*)\big) + (\rho + \rho^\server) c_T \sum_{t=0}^{T-1} \del_t \bb E\|\nabla f(\bs x_t)\|\\
&\ts \qquad + \frac{1}{2} L c_T \sum_{t=0}^{T-1} \del_t^2 \|\bs v_t\|^2.
\end{align*}
Therefore, if we further assume that $\|\nabla f\| \leq B$ (\eg in a certain region of space)---which is by the way involved by imposing $\max_n \|\bs g_n\| \leq B$ as earlier---, then 
\begin{align*}
&\ts c_T \sum_{t=0}^{T-1} \del_t \bb E \|\nabla f(\bs x_t)\|^2\\
&\ts \leq c_T \big(f(\bs x_0) - f(\bs x_*)\big) + (\rho + \rho^\server) B + \frac{1}{2} L c_T \sum_{t=0}^{T-1} \del_t^2 \|\bs v_t\|^2.
\end{align*}
Here also, taking $\alpha^\server, \alpha = \Omega((\log T)^q)$ ensures that $(\rho + \rho^\server) B$ be arbitrarily small.  The final bound is reached by considering that $\|\bs v_t\|^2 \leq (\lambda^\server)^2 d$.
\end{proof}

}

\appendix

\section{Auxiliary Results}
This first lemma is central to the characterization of the general algorithm presented in Sec.~\ref{sec:conv_guar}, as well as in the proof of Prop.~\ref{prop:convex:quantized_SGD:distributed}.

\begin{lemma}\cite[Adapted from Lemma~14.1]{SSBD14}\label{lem:alg_fact}
Given an initial vector $\bs x_0 \in \bb R^d$, a set of directions $\{\bs v_t\}_{t=0}^T \subset \bb R^d$, step sizes $\{\del_t\}_{t=0}^T \subset \bb R_+$, and the relation
\begin{equation}\label{eq:alg_fact}
    \bs x_{t+1} = \bs x_t - \del_t \bs v_t, \quad t=0, \ldots, T, 
\end{equation}
the following holds for any $\bs x_{\ast}\in \R^d$, 
\begin{equation*}
\ts \sum_{t=0}^T \del_t \scp{\bs x_t-\bs x_{\ast}}{\bs v_t}\ \leq\ \frac{1}{2}\eu{\bs x_{\ast}-\bs x_0}^2 +\frac{1}{2} \sum_{t=0}^T \del_t^2 \eu{\bs v_t}^2.  
\end{equation*}
\end{lemma}

The next elements provide some useful facts about the mean of independent random vectors with bounded variances.
We say that a random vector $\bs u$ has variance bounded by $\sigma^2$ if 
$\E \big[ \eu{\bs u - \E \bs u}^2\big]\leq \sigma^2$.

\begin{lemma}\label{lem:var_decrease}
Assume that $\bs u_1,\ldots, \bs u_N\in \R^d$ are independent vectors, where $\bs u_i$ has variance bounded by $\sigma_i^2$ for $i=1,\ldots, N$. Then their mean $\bar{\bs u}:=\tfrac{1}{N}\sum_{i=1}^N \bs u_i$ has variance bounded by $\sigma^2:=\tfrac{1}{N^2}\sum_{i=1}^N\sigma_i^2$. 
\end{lemma}
\begin{proof} Set $\bs y_i=\bs u_i-\E \bs u_i$. Then
$$
\ts \E \big[ \|\bar{\bs u} - \E \bar{\bs u}\|_2^2\big]=
\E \Big[ \big\|\tfrac{1}{N}\sum_{i=1}^N \bs y_i\big\|_2^2\Big]=\tfrac{1}{N^2}\E \big[ \big\|\sum_{i=1}^N \bs y_i\big\|_2^2\big]
$$
and 
$$
\ts \E \big[ \big\|\sum_{i=1}^N \bs y_i\big\|_2^2\big] =  \sum_{i,j=1}^N \E \scp{\bs y_i}{\bs y_j} =  \sum_{i=1}^N \E \eu{\bs y_i}^2\leq \sum_{i=1}^N \sigma_i^2.
$$
\end{proof}
\begin{corollary}\label{coro:var_decrease}
Assume that $\bs u_1,\ldots, \bs u_N\in \R^d$ are independent vectors each with variance bounded by $\sigma^2$. Then their mean $\bar{\bs u}:=\tfrac{1}{N}\sum_{i=1}^N \bs u_i$ has variance bounded by $\sigma^2/N$. 
\end{corollary}

The next three lemmas show that we can characterize analytically the expectation of a few key functions of our one-bit quantizer (such as its distance to its unquantized input, or its scalar product with another vector) relatively to their random dithering.

\begin{lemma}\label{lem:K_bit_quantizer} For $\lambda>0, K\in \N$, let $\dithQ_{\lambda, K}:\R^d \to \R^d$ be a random quantizer according to Def.~\ref{def:K_bit_quantizer}. 
\begin{enumerate}
\item[1)]
For any $\bs x\in \R^d$, 
\begin{equation}\label{eq:lem:K_bit_quantizer:statement1}
\E \big[\dithQ_{\lambda,K}(\bs x)\big] = \big(x_k 1_{x_k\in [-\lambda,\lambda]} +\lambda 1_{x_k>\lambda} + (-\lambda)1_{x_k<-\lambda}\big)_{k=1}^d.
\end{equation}
In particular, if $\|\bs x\|_\infty\leq \lambda$, then 
$\E \big[\dithQ_{\lambda,K}(\bs x)\big]=x$. 
\item[2)]
For any $\bs x\in \R^d$,
\begin{equation*}
\bs x - \E \big[\dithQ_{\lambda, K}(\bs x)\big] = \big(\sgn(x_k)(|x_k|-\lambda)1_{|x_k|>\lambda}\big)_{k=1}^d.
\end{equation*}
The expression on the right hand side is the soft-thresholding of $\bs x$ with thresholding $\lambda$.
\item[3)]
For any $\bs x\in \R^d$ with $\|\bs x\|_\infty\leq \lambda$,
\begin{equation*}
\E \big[\eu{\dithQ_{\lambda, K}(\bs x) - \bs x}^2\big] =  \tfrac{1}{K}(\lambda^2 d - \eu{\bs x}^2).
\end{equation*}
\end{enumerate}
\end{lemma}
\begin{proof}
Statement~$2$ immediately follows from statement~$1$. To show statement~$1$, first observe that $\E \big[\dithQ_{\lambda,K}(\bs x)\big]=\E \big[\dithQ_{\lambda}(\bs x)\big]$, where $\dithQ_{\lambda}(\bs x)=\lambda \sgn(\bs x+ \bs \tau)$ with $\bs \tau\sim\mathcal{U}([-\lambda, \lambda]^d)$ is a random one-bit quantizer according to Def.~\ref{def:K_bit_quantizer}. We show \eqref{eq:lem:K_bit_quantizer:statement1} coordinate-wise. 
Let $i\in [d]$. If $x_i>\lambda$, then $x_i+\tau_i>0$ which implies
$\big(\dithQ_{\lambda}(\bs x)\big)_i=\lambda \sgn(x_i+\tau_i)=\lambda$. Therefore $\E\big[\big(\dithQ_{\lambda}(\bs x)\big)_i\big]=\lambda$ in this case. Analogously, we have 
$\E\big[\big(\dithQ_{\lambda}(\bs x)\big)_i\big]=-\lambda$ if $x_i<-\lambda$. Next, let $x_i\in[-\lambda, \lambda]$. Then 
\begin{align*}
\E\big[\big(\dithQ_{\lambda}(\bs x)\big)_i\big] = \E\big[\lambda\sgn(x_i+\tau_i)\big]&= \lambda\cdot\Pb(x_i+\tau_i\geq 0) - \lambda\cdot\Pb(x_i+\tau_i< 0)\\
&= \lambda\cdot\Pb(\tau_i\geq - x_i) - \lambda\cdot\Pb(\tau_i< - x_i)\\
&= \lambda\cdot \big(\tfrac{\lambda + x_i}{2\lambda}\big) - \lambda\cdot\big(\tfrac{- x_i+\lambda}{2\lambda}\big)\\
&= x_i.
\end{align*}
This shows \eqref{eq:lem:K_bit_quantizer:statement1}. To show 
statement~$3$, let 
\begin{equation*}
\dithQ_{\lambda, K}(\bs x)=\tfrac{\lambda}{K}\sum_{i=1}^K\sgn(\bs x+\bs \tau_i)
\end{equation*}
for independent and uniformly distributed random vectors $\bs \tau_1, \ldots, \bs \tau_K\in [-\lambda, \lambda]^d$. 
Using $\dithQ_{\lambda, K}(\bs x) - \bs x=\tfrac{1}{K}\sum_{k=1}^K(\lambda\sgn(\bs x+\bs \tau_k)-\bs x)$, we obtain 
\begin{equation*}
\E \big[\eu{\dithQ_{\lambda, K}(\bs x) - \bs x}^2\big] = \tfrac{1}{K^2}\sum_{k, l=1}^K\E \big[\langle \lambda\sgn(\bs x+\bs \tau_k)-\bs x, \lambda\sgn(\bs x+\bs \tau_l)-\bs x\rangle\big].
\end{equation*}
By the independence of $\bs \tau_1, \ldots, \bs \tau_K$ and equation \eqref{eq:lem:K_bit_quantizer:statement1}, we have
\begin{equation*} 
\E \big[\langle \lambda\sgn(\bs x+\bs \tau_k)-\bs x, \lambda\sgn(\bs x+\bs \tau_l)-\bs x\rangle\big] =0
\end{equation*}
for all $k\neq l$. Therefore, 
\begin{equation}\label{eq:lem:K_bit_quantizer:1}
\E \big[\eu{\dithQ_{\lambda, K}(\bs x) - \bs x}^2\big] = \tfrac{1}{K^2}\sum_{k=1}^K\E \big[\big\|\lambda\sgn(\bs x+\bs \tau_k)-\bs x\big\|_2^2\big].
\end{equation}
Let $k\in [K]$ and let $\tau_{k,i}$ denote the $i$-th entry of $\bs \tau_k$. We have 
\begin{align}\label{eq:lem:K_bit_quantizer:2}
\E \big[\big\|\lambda\sgn(\bs x+\bs \tau_k)-\bs x\big\|_2^2\big] &=\ts 
\sum_{i=1}^d \E  \big(\lambda\sgn(x_i+\tau_{k,i})-x_i\big)^2\\ \nonumber
&=\ts \sum_{i=1}^d \big(\lambda^2 -2 x_i \E[\lambda\sgn( x_i+\tau_{k,i})] + x_i^2\big)\\ \nonumber
&=d\lambda^2 - \eu{\bs x}^2,
\end{align}
where we have used that $\E[\lambda\sgn(x_i+\tau_{k,i})]=\bs x_i$ by \eqref{eq:lem:K_bit_quantizer:statement1}. The result now follows from \eqref{eq:lem:K_bit_quantizer:1}
and \eqref{eq:lem:K_bit_quantizer:2}.
\end{proof}

\begin{lemma}\label{lem:help}
For $\lambda>0$, let $\dithQ_{\lambda}:\R^d \to \R^d$ be a random one-bit quantizer according to Def.~\ref{def:K_bit_quantizer}. For any $\bs x\in \R^d$ and any $\bs y\in \R^d$ with $\|\bs y\|_{\infty}\leq \lambda$,  
\begin{equation*}
\ts \E \big[\scp{\bs x}{\dithQ_\lambda(\bs y)}^2\big]=\scp{\bs x}{\bs y}^2 + \eu{\bs x}^2\lambda^2 - \sum_{i=1}^d x_i^2y_i^2.
\end{equation*}
\end{lemma}
\begin{proof} 
Let $\dithQ_{\lambda}(\bs y)=\lambda\sgn(\bs y+\bs \tau)$ for $\bs \tau\sim \mathcal{U}([-\lambda, \lambda]^d)$. Using the independence of the coordinates $\tau_1,\ldots, \tau_d$ of $\bs \tau$ and Lemma~\ref{lem:K_bit_quantizer} we obtain  
\begin{align*}
\E \big[\scp{\bs x}{\dithQ_\lambda(\bs y)}^2\big]
&=\ts \sum_{i,j=1}^d x_i x_j \E\big[\lambda\sgn(y_i+\tau_i) \lambda\sgn(y_j+\tau_j)\big]\\
&=\ts \sum_{i\neq j} x_i x_j y_i y_j + \sum_{i=1}^d x_i^2\lambda^2\\
&=\ts \sum_{i, j=1}^d x_i x_j y_i y_j + \sum_{i=1}^d (x_i^2\lambda^2 - x_i^2y_i^2).
\end{align*}
\end{proof}

\begin{lemma}\label{lem:var}
For $\lambda>0$, let $\dithQ_{\lambda}:\R^d \to \R^d$ be a random one-bit quantizer according to Def.~\ref{def:K_bit_quantizer}.
For any matrix $\bs A\in \R^{k\times d}$ and any vector 
$\bs x\in \R^d$ with $\|\bs x\|_{\infty}\leq \lambda$,  
\begin{equation*}
\ts \E \big[\eu{A\dithQ_\lambda(\bs x) - \bs A\bs x}^2\big]= \|\bs A\|_F^2\lambda^2 - \sum_{i=1}^k\sum_{j=1}^d A_{i,j}^2 x_j^2.
\end{equation*}
\end{lemma}
\begin{proof} Let $\bs a^{(i)}$ denote the $i$-th row of $\bs A$. Since $\|\bs x\|_\infty\leq \lambda$, Lemma~\ref{lem:K_bit_quantizer} implies  
$\E[\dithQ_\lambda(\bs x)]=x$. Therefore, 
\begin{multline}\label{eq:lem:var:1}
\E \big[\eu{\bs A\dithQ_\lambda(\bs x) - \bs A\bs x}^2\big]
=\E \big[\eu{\bs A\dithQ_\lambda(\bs x)}^2\big] + \eu{\bs A\bs x}^2
 - 2 \scp{\bs A\, \E[\dithQ_\lambda(\bs x)]}{\bs A\bs x}\\
=\E \big[\eu{\bs A\dithQ_\lambda(\bs x)}^2\big] - \eu{\bs A\bs x}^2. 
\end{multline}
Lemma~\ref{lem:help} implies 
\begin{align}\label{eq:lem:var:2}
\E \big[\eu{\bs A\dithQ_\lambda(\bs x)}^2\big] &\ts = \sum_{i=1}^k \E \big[\scp{\bs a^{(i)}}{\dithQ_\lambda(\bs x)}^2\big]\\ \nonumber
&\ts =\sum_{i=1}^k \Big(\scp{\bs a^{(i)}}{\bs x}^2 + \eu{\bs a^{(i)}}^2\lambda^2 - \sum_{j=1}^d A_{i,j}^2 x_j^2\Big)\\ \nonumber
&\ts = \eu{\bs A \bs x}^2 + \|\bs A\|_F^2\lambda^2 - \sum_{i=1}^k\sum_{j=1}^d A_{i,j}^2 x_j^2.
\end{align}
The result now follows from \eqref{eq:lem:var:1} and \eqref{eq:lem:var:2}.
\end{proof}

The following result is well-known in the case where $\bs \Htransform$ is an orthonormal Hadamard matrix (\eg see \cite[Lemma B.2]{Oym16}). Since our version of the result is slightly more general, we included a proof below. 

\begin{lemma}\label{lem:Hadamard_general}
Let $\bs \eps\in \{-1,1\}^d$ be a uniformly distributed random vector and $\bs \Htransform\in \R^{d\times d}$
a matrix with $|\Htransform_{i,j}|\leq \gamma$ for all $i,j\in [d]$. 
For any $\bs x\in \R^d$ and $t\geq 0$,  
\begin{equation*}
\|\bs H\bs D_{\bs \eps} \bs x\|_\infty\leq t \eu{\bs x}
\end{equation*}
with probability at least $1-2d\exp(-t^2/2\gamma^2)$.
\end{lemma}
\begin{proof}
We may assume $\eu{\bs x}=1$.
Let $\bs h^{(i)}$ denote the $i$-th row of $\bs \Htransform$ and set $\tilde{\bs x}_i=\bs D_{\bs x} \bs h^{(i)}{}^\top$. 
By Hoeffding's inequality for Rademacher sums (see, \eg \cite[Corollary 7.21.]{foucart2013cs}), 
\begin{equation*}
\Pb(|\scp{\tilde{\bs x}_i}{\bs \eps}|\geq t)\leq 2\exp(-t^2/2\eu{\tilde{\bs x}_i}^2)
\end{equation*}
for all $t> 0$.
Since  
$\eu{\tilde{\bs x}_i}\leq \gamma\eu{\bs x}$ and 
$(\bs H\bs D_{\bs \eps} \bs x)_i=\scp{\tilde{\bs x}_i}{\bs \eps}$, this implies for any $t>0$, 
\begin{equation*}
\Pb(|(\bs H\bs D_{\bs \eps} \bs x)_i|\geq t)\leq 2\exp(-t^2/2\gamma^2).
\end{equation*}
Applying the union bound, we obtain
\begin{equation*}
\Pb(\|\bs H\bs D_{\bs \eps} \bs x\|_\infty\geq t)
\leq 2d\exp(-t^2/2\gamma^2),
\end{equation*}
which shows the result.
\end{proof}

\begin{proof}[Proof of Lemma~\ref{lem:Hadamard}]
It holds $\bs \Htransform_{\bs \eps}=\bs H \bs D_{\bs \eps}$ for $\bs \Htransform\in \R^{d\times d}$ a universal sensing basis and $\bs \eps\in \{-1,1\}^d$ a uniformly distributed random vector. Since $\bs \Htransform\in \R^{d\times d}$ is a universal sensing basis, it holds $|\bs \Htransform_{i,j}|=\tfrac{1}{\sqrt{d}}$ for all $i,j\in [d]$. Hence, Lemma~\ref{lem:Hadamard_general} implies that for every $t\geq 0$,
\begin{equation}
\Pb(\|\bs \Htransform_{\bs \eps} \bs x\|_\infty> t\eu{\bs x})\leq 2d\exp(-t^2d/2).
\end{equation}

Observe that 
\begin{align*}
d\exp(-\tfrac{1}{2}t^2d)&=\exp(\log d -\tfrac{1}{2}t^2d)
\leq \exp(-\tfrac{1}{4}t^2d)
\end{align*}
for all $t\geq 2\sqrt{\tfrac{\log d}{d}}$. Substituting $t=\alpha \sqrt{\tfrac{\log d}{d}}$, we obtain the result.
\end{proof}

{The next result leverages the previous corollary to characterize the deviation between the average of multiple independent $\lambda$-transforms, each applied to different vectors, and the mean of these vectors. This element is central to the proof of Prop.~\ref{prop:convex:quantized_SGD:distributed}.}

\begin{corollary}\label{lem:Hadamard:average}
There exists an absolute constant $C\geq 2$ such that the following holds.
Let $\Lfct_1,\ldots, \Lfct_N$ be independent copies of the random $\lambda$-transform 
$\Ztransform_{\lambda}:\R^d\to \R^d$ from Def.~\ref{def:K_bit_quantizer:flattened}.
For $B>0$ and $\alpha\geq C$, set 
\begin{equation*}
\lambda = \alpha B \sqrt{\tfrac{\log d}{d}}.
\end{equation*}
Then for any vectors $\bs x_1, \ldots, \bs x_N\in \R^d$ with $\max_{n\in [N]}\|\bs x_n\|_2\leq B$,
\begin{multline*}
\ts \E \Big[ \Big\|\tfrac{1}{N}\sum_{n=1}^N \Lfct_{n}(\bs x_n) - \tfrac{1}{N}\sum_{n=1}^N \bs x_n\Big\|_2^2\Big] \leq \tfrac{1}{N^2}\sum_{n=1}^N \E \big[\eu{\Lfct_{n}(\bs x_n) - \bs x_n}^2
\big]\\ 
\quad + 8N B^2 \cdot \exp(-\tfrac{1}{8}\alpha^2\log d). 
\end{multline*}
\end{corollary}
\begin{proof}
For $n\in [N]$, let  
\begin{equation*}
\Lfct_{n}(\bs x)=\bs \Htransform_{\bs \eps_n}^{-1} (\dithQ_{\lambda})_n(\bs \Htransform_{\bs \eps_n} \bs x)
\end{equation*}
and
\begin{equation*}
(\dithQ_{\lambda})_n(\bs x) = \lambda\sgn(\bs x+\bs \tau_n)
\end{equation*}
where $\bs \eps_1, \ldots, \bs \eps_N\in \{-1,1\}^d$ and $\bs \tau_1, \ldots, \bs \tau_N\in [-\lambda, \lambda]^d$ are independent and uniformly distributed random vectors. 
Set $\tilde{\bs x}_n=\bs \Htransform_{\bs \eps_n}\bs x_n$ for $n\in [N]$.
Then $\Lfct_{n}(\bs x_n) = \bs \Htransform_{\bs \eps_n}^{-1}  (\dithQ_{\lambda})_n(\tilde{\bs x}_n)$.
Define the events 
\begin{equation*}
\mathcal{A}_n := \{\|\tilde{\bs x}_n\|_\infty\leq \lambda\}, \qquad\mathcal{A}:= \cap_{n=1}^N \mathcal{A}_n.
\end{equation*}
{Define the sets} $\bs \tau := \{\bs \tau_1, \ldots, \bs \tau_N\}$ and $\bs \eps:=\{\bs \eps_1, \ldots, \bs \eps_N\}$. 
We have
\begin{multline*}
\ts \E \Big[ 1_{\mathcal{A}}\cdot \Big\|\sum_{n=1}^N \Lfct_{n}(\bs x_n) - \sum_{n=1}^N \bs x_n\Big\|_2^2\Big] = \E \Big[ 1_{\mathcal{A}}\cdot \Big\| \sum_{n=1}^N \bs \Htransform_{\bs \eps_n}^{-1}\big((\dithQ_{\lambda})_n(\tilde{\bs x}_n) - \tilde{\bs x}_n\big) \Big\|_2^2\Big]\\
\ts = \E_{\bs \eps}\Big[1_{\mathcal{A}}\cdot  
\sum_{n, m=1}^N \E_{\bs \tau}\big[\scp{\bs \Htransform_{\bs \eps_n}^{-1} \big((\dithQ_{\lambda})_n(\tilde{\bs x}_n) - \tilde{\bs x}_n\big)}{\bs \Htransform_{\bs \eps_m}^{-1} \big((\dithQ_{\lambda})_m(\tilde{\bs x}_m) - \tilde{\bs x}_m\big)}
\big]\Big].
\end{multline*}
Lemma~\ref{lem:K_bit_quantizer} implies that on the event $\mathcal{A}$ it holds $\E_{\bs \tau_n}\big[(\dithQ_{\lambda})_n(\tilde{\bs x}_n)\big]=\tilde{\bs x}_n$ for all $n=1, \ldots, N$. In combination with the independence of $\tau_1, \ldots, \tau_N$ this implies that for all $n\neq m$, 
\begin{equation*}
\E_{\bs \tau}\big[\scp{\bs \Htransform_{\bs \eps_n}^{-1} \big((\dithQ_{\lambda})_n(\tilde{\bs x}_n) - \tilde{\bs x}_n\big)}{\bs \Htransform_{\bs \eps_m}^{-1}\big((\dithQ_{\lambda})_m(\tilde{\bs x}_m) - \tilde{\bs x}_m\big)}
\big]\Big] = 0.
\end{equation*}
Therefore
\begin{multline*}
\ts \E \Big[ 1_{\mathcal{A}}\cdot \Big\|\sum_{n=1}^N \Lfct_{n}(\bs x_n) - \sum_{n=1}^N \bs x_n\Big\|_2^2\Big]\\
\ts = \E_{\bs \eps}\Big[1_{\mathcal{A}}\cdot  
\sum_{n=1}^N \E_{\bs \tau}  \big[\eu{\bs \Htransform_{\bs \eps_n}^{-1} \big((\dithQ_{\lambda})_n(\tilde{\bs x}_n) - \tilde{\bs x}_n\big)}^2
\big]\Big]\\
\ts \leq \sum_{n=1}^N \E \big[\eu{\Lfct_n(\bs x_n)-\bs x_n}^2
\big].
\end{multline*}
To estimate the expectation of $\big\|\sum_{n=1}^N \Lfct_{n}(\bs x_n) - \sum_{n=1}^N \bs x_n\big\|_2^2$ on $\mathcal{A}^C$, first notice that 
we have the trivial estimate 
\begin{align*}
\ts \Big\|\sum_{n=1}^N \Lfct_{n}(\bs x_n) - \sum_{n=1}^N \bs x_n\Big\|_2^2 
&\ts \leq 2 \Big\|\sum_{n=1}^N \Lfct_{n}(\bs x_n) \Big\|_2^2
+ 2 \Big\|\sum_{n=1}^N \bs x_n\Big\|_2^2\\
&\ts \leq 2N^2\lambda^2 d+2N^2B^2\\
&\ts \leq 4N^2\lambda^2 d,
\end{align*}
where we have used that $\lambda \sqrt{d}\geq B$ and $\eu{\Lfct_n(\bs x_n)}\leq \lambda\sqrt{d}$ for all $n\in [N]$.
Further, by the union bound and Cor.~\ref{lem:Hadamard}, 
\begin{align*}
\ts \Pb(\mathcal{A}^C)\leq\sum_{n=1}^N \Pb(\|\tilde{\bs x}_n\|_\infty >\lambda)
&\ts \leq 
\sum_{n=1}^N \Pb(\|\bs \Htransform_{\bs \eps_n}\bs x_n\|_\infty > \alpha\sqrt{\tfrac{\log d}{d}} \eu{\bs x_n})\\
&\ts \leq 2N\exp(-\tfrac{1}{4}\alpha^2\log d).
\end{align*}
Putting everything together, we obtain 
\begin{align*}
&\ts \E \big[\big\|\tfrac{1}{N}\sum_{n=1}^N \Lfct_{n}(\bs x_n) - \tfrac{1}{N}\sum_{n=1}^N \bs x_n\big\|_2^2\big]\\
&\ts =\E \big[ 1_{\mathcal{A}}\cdot \big\|\tfrac{1}{N}\sum_{n=1}^N \Lfct_{n}(\bs x_n) - \tfrac{1}{N}\sum_{n=1}^N \bs x_n\big\|_2^2\big]\\
&\ts \qquad\qquad + \E \big[ 1_{\mathcal{A}^C}\cdot \big\|\tfrac{1}{N}\sum_{n=1}^N \Lfct_{n}(\bs x_n) - \tfrac{1}{N}\sum_{n=1}^N \bs x_n\big\|_2^2\big]\\
&\ts \leq \tfrac{1}{N^2}\sum_{n=1}^N \E \big[\eu{\Lfct_n(\bs x_n)-\bs x_n}^2
\big]
 + 8\lambda^2 d N\exp(-\tfrac{1}{4}\alpha^2\log d).
\end{align*} 
Finally, observe that
\begin{align}\label{eq:std_calc}
\lambda^2 d \cdot \exp(-\tfrac{1}{4}\alpha^2\log d) &= B^2 \cdot (\alpha^2 \log d) \cdot \exp(-\tfrac{1}{4}\alpha^2\log d)\\ \nonumber
&\leq B^2 \cdot \exp(-\tfrac{1}{8}\alpha^2\log d),
\end{align}
for all $\alpha\geq C$, where $C>0$ denotes an absolute constant that is chosen large enough. This shows the result.
\end{proof}

{Finally, we use this lemma in the proof of Thm.~\ref{thm:Hadamard:all_compressed}.}

\begin{lemma}\label{lem:K:rad} 
There exists an absolute constant $C\geq 2$ such that the following holds.
For $\lambda>0, K\in \N$, let $\Ztransform_{\lambda, K}:\R^d \to \R^d$ be a $(\lambda, K)$-transform according to Def.~\ref{def:K_bit_quantizer:flattened}.
For $B>0$ and $\alpha\geq C$, set 
$\lambda = \alpha B \sqrt{\tfrac{\log d}{d}}$.
Then for any $\bs x\in \R^d$ with $\eu{\bs x}\leq B$,
\begin{equation*}
\E \big[\big\|\Ztransform_{\lambda, K}(\bs x) - \bs x\big\|_2^2\big] \leq \tfrac{\alpha^2 B^2\log d}{K}
 + 8B^2 \exp(-\tfrac{1}{8}\alpha^2 \log d).
\end{equation*}
\end{lemma}
\begin{proof}
Let 
\begin{equation*}
\Ztransform_{\lambda, K}(\bs x) = \bs \Htransform_\eps^{-1} \dithQ_{\lambda, K}(\bs \Htransform_\eps \bs x)
\end{equation*}
and
\begin{equation*}
\dithQ_{\lambda, K}(\bs x) = \tfrac{\lambda}{K}\sum_{k=1}^K\sgn(\bs x+\bs \tau_k)
\end{equation*}
for $\bs \eps\in \{-1,1\}^d$ and $\bs \tau_1, \ldots, \bs \tau_K\in [-\lambda, \lambda]^d$ independent and uniformly distributed random vectors. 
Set $\tilde{\bs x}=\bs \Htransform_{\bs \eps} \bs x$ and define the event 
\begin{equation*}
\mathcal{A}:=\big\{\|\tilde{\bs x}\|_\infty\leq \lambda\big\}.
\end{equation*}
Define the set $\bs \tau :=\{\bs \tau_1,\ldots, \bs \tau_K\}$, the matrix $\bs V:= \bs \Htransform_\eps^{-1}$, and let $\bs v^{(i)}$ denote the $i$-th row of $\bs V$. 
On the event $\mathcal{A}$ we have
\begin{align*}
&\ts \E_{\bs \tau} \big[\big\|\Ztransform_{\lambda, K}(\bs x) - \bs x\big\|_2^2\big] =\E_{\bs \tau} \big[\big\|V\dithQ_{\lambda, K}(\tilde{\bs x}) - \bs V\tilde{\bs x}\big\|_2^2\big]\\ 
&\ts = \sum_{i=1}^d \E_{\bs \tau}\Big[\Big(\big\langle \bs v^{(i)}, \tfrac{1}{K}\sum_{k=1}^K\big(\lambda\sgn(\tilde{\bs x}+\bs \tau_k) -\tilde{\bs x}\big)\big\rangle \Big)^2\Big]\\
&=\tfrac{1}{K^2}\sum_{i=1}^d \sum_{k,k'=1}^K
\E_{\bs \tau}\big[\big(\big\langle \bs v^{(i)}, \lambda\sgn(\tilde{\bs x}+\bs \tau_k) -\tilde{\bs x}\big\rangle\big)\cdot (\big\langle \bs v^{(i)}, \lambda\sgn(\tilde{\bs x}+\bs \tau_{k'}) -\tilde{\bs x}\big\rangle\big)\big]\\
&\ts =\tfrac{1}{K^2}\sum_{i=1}^d \sum_{k=1}^K
\E_{\bs \tau}\big[\big(\big\langle \bs v^{(i)}, \lambda\sgn(\tilde{\bs x}+\bs \tau_k) -\tilde{\bs x}\big\rangle\big)^2\big]\\
&\ts =\tfrac{1}{K} \E_{\bs \tau}\big[ \eu{\bs V\dithQ_{\lambda}(\tilde{\bs x}) - \bs V\tilde{\bs x}}^2\big],
\end{align*}
where $\dithQ_{\lambda}=\dithQ_{\lambda,1}$ is the one-bit quantizer from  Def.~\ref{def:K_bit_quantizer}.
Observe that for the second to last equality we have used that the random vectors $\bs \tau_1, \ldots, \bs \tau_K$ are independent and that by Lemma~\ref{lem:K_bit_quantizer} we have $\E_{\bs \tau_k}[\lambda\sgn(\tilde{\bs x}+\bs \tau_k)]=\tilde{\bs x}$ 
on the event $\mathcal{A}$ for $k=1,\ldots, K$.
Lemma~\ref{lem:var} implies that on the event $\mathcal{A}$,
\begin{equation*}
\E_{\bs \tau}\big[ \eu{\bs V\dithQ_{\lambda}(\tilde{\bs x}) - \bs V\tilde{\bs x}}^2\big]\leq\|\bs V\|_F^2\lambda^2.
\end{equation*}
Therefore, 
\begin{align*}
\E \big[\big\|\Ztransform_{\lambda, K}(\bs x) - \bs x\big\|_2^2\big]
&=\E \big[\big\|\bs V\dithQ_{\lambda, K}(\tilde{\bs x}) - \bs V\tilde{\bs x}\big\|_2^2\big]\\ 
&= \E_{\bs \eps} \big[1_{\mathcal{A}}\E_{\bs \tau}\big[\big\|\bs V\dithQ_{\lambda, K}(\tilde{\bs x}) - \bs V\tilde{\bs x}\big\|_2^2\big]\big]\\
&\quad + \E_{\bs \eps} \big[1_{\mathcal{A}^C}\E_{\bs \tau}\big[\big\|\bs V\dithQ_{\lambda, K}(\tilde{\bs x}) - \bs V\tilde{\bs x}\big\|_2^2\big]\big]\\
&\leq \tfrac{\lambda^2}{K}\E_{\bs \eps} \big[\|\bs V\|_F^2\big]
 + \E_{\bs \eps} \big[1_{\mathcal{A}^C}\E_{\bs \tau}\big[\big\|\bs V\dithQ_{\lambda, K}(\tilde{\bs x}) - \bs V\tilde{\bs x}\big\|_2^2\big]\big].
\end{align*}
Using $\bs V=\bs D_{\bs \eps} \bs \Htransform^\top$, it is easy to check that  
$\|\bs V\|_F^2=d$. Further, 
\begin{align*}
\big\|\bs V\dithQ_{\lambda, K}(\tilde{\bs x}) - \bs V\tilde{\bs x}\big\|_2
= \big\|\Ztransform_{\lambda, K}(\bs x) - \bs x\big\|_2\leq \eu{\Ztransform_{\lambda, K}(\bs x)} + B\leq \lambda\sqrt{d} + B\leq 2\lambda\sqrt{d},
\end{align*}
where we have used that $\eu{\Ztransform_{\lambda, K}(\bs x)}\leq \lambda\sqrt{d}$ for all $\bs x\in \R^d$ (see Remark~\ref{rem:simple}).
It remains to bound $\Pb(\mathcal{A}^C)$.
Using Cor.~\ref{lem:Hadamard} and the assumption $\eu{\bs x}\leq B$, we obtain  
\begin{align*}
\Pb(\mathcal{A}^C)\leq \Pb(\|\bs \Htransform_{\bs \eps} \bs x\|_\infty >\alpha \sqrt{\tfrac{\log d}{d}}\eu{\bs x})
&\leq 2\exp(-\tfrac{1}{4}\alpha^2\log d).
\end{align*}
Putting everything together, we obtain  
\begin{equation*}
\E \big[\big\|\Ztransform_{\lambda, K}(\bs x) - \bs x\big\|_2^2\big]
\leq \tfrac{\lambda^2 d}{K}
 + 8\lambda^2 d \exp(-\tfrac{1}{4}\alpha^2\log d).
\end{equation*}
The result now follows using inequality \eqref{eq:std_calc}.
\end{proof}

\bibliographystyle{amsplain}
\bibliography{references}

\end{document}